\documentclass[10pt,twocolumn,letterpaper]{article}

\usepackage{cvpr}              
\usepackage{gensymb}
\usepackage{amssymb}
\usepackage{multirow}
\usepackage{graphicx}
\usepackage{color}
\usepackage{float}
\usepackage{ulem}

\usepackage{algorithm}
\usepackage{algorithmic}

\usepackage{amsmath, amssymb, amsthm}
\usepackage{mathtools}
\newtheorem{definition}{Definition}
\newtheorem{proposition}{Proposition}
\newtheorem{theorem}{Theorem}

\usepackage{caption}
\definecolor{cvprblue}{rgb}{0.21,0.49,0.74}
\usepackage[table]{xcolor}

\definecolor{best}{RGB}{198,239,206}     
\definecolor{second}{RGB}{226,239,218}   
\definecolor{third}{RGB}{245,248,225}    
\newcommand{\method}{SOAR}
\definecolor{cvprblue}{rgb}{0.21,0.49,0.74}
\usepackage[pagebackref,breaklinks,colorlinks,allcolors=cvprblue]{hyperref}

\def\confName{CVPR}
\def\confYear{2026}

\title{SOAR: Regre\textcolor{blue}{s}si\textcolor{blue}{o}n-based LiD\textcolor{blue}{A}R \textcolor{blue}{R}elocalization for UAVs}

\author{
    Hengyu Mu$^{1,2,*}$, Jianshi Wu$^{1,2,*}$, Yuxin Guo$^{1,2,*}$, XianLian Lin$^{1,2}$, Qingyong Hu$^{3}$, \\Sheng Ao$^{1,2}$, Chenglu Wen$^{1,2}$, Cheng Wang$^{1,2}$\\
    $^{1}$ Fujian Key Laboratory of Sensing and Computing for Smart Cities, Xiamen University\\
    $^{2}$ Key Laboratory of Multimedia Trusted Perception and Efficient Computing, \\Ministry of Education of China, Xiamen University\\
    $^{3}$ Department of Computer Science at the University of Oxford\\
    \{23020250157845, wujianshi, guoyuxin\}@stu.xmu.edu.cn, lxl@xmu.edu.cn, \\huqingyong15@outlook.com, \{aosh, clwen, cwang\}@xmu.edu.cn
}

\begin{document}
\twocolumn[{%
\maketitle

\begin{figure}[H]
\hsize=\textwidth 
\centering
\includegraphics[width=1.0\textwidth]{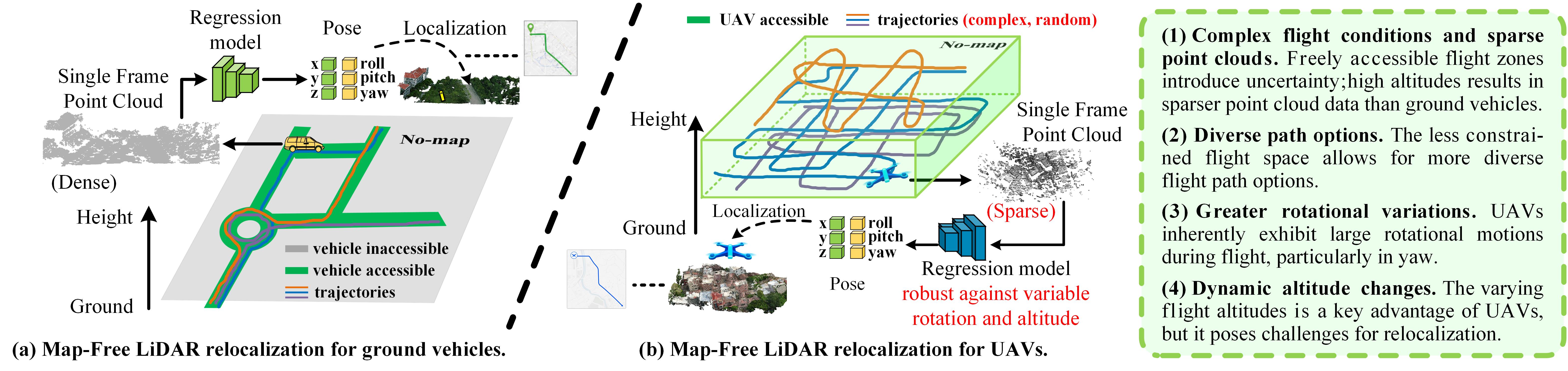}
\caption{The challenges of UAV's map-free relocalization compared to vehicle. Compared to map-free vehicle relocalization, UAV relocalization faces greater challenges.}
\label{fig_1}
\end{figure}
}]

\begingroup
\renewcommand\thefootnote{\fnsymbol{footnote}}
\footnotetext[1]{Equal contribution.}
\endgroup


\begin{abstract}

Regression-based LiDAR relocalization has recently emerged as a promising solution for high-precision positioning in GNSS-denied environments. However, these methods are primarily tailored to autonomous driving, exhibiting significantly degraded accuracy in unmanned aerial vehicle (UAV) scenarios due to arbitrary pose variations and irregular flight paths. In this paper, we propose \method{}, a regre\textbf{s}si\textbf{o}n-based LiD\textbf{A}R \textbf{r}elocalization framework for UAVs. Specifically, we introduce a locality-preserving sliding window attention module with locally invariant positional encoding to capture discriminative geometric structures robust to viewpoint changes. A coordinate-independent feature initialization module is further designed to eliminate sensitivity to global transformations. Furthermore, most existing UAV datasets are limited to evaluate LiDAR relocalization in real-world, due to the lack of synchronized LiDAR scans, accurate 6-DoF poses, or multiple traversals. Thus, we construct a large-scale UAV LiDAR localization dataset with 4 scenes and 13 irregular paths exhibiting rotation and altitude variations, providing a more realistic benchmark for UAVs. Extensive experiments demonstrate that our method achieves state-of-the-art performance, improving the localization success rate by \textbf{40\%} and reducing mean error \textbf{over 10m} on UAVLoc. Our code and dataset will be released soon

\end{abstract}    
\section{Introduction}
\label{sec:intro}

\begin{figure*}
  \centering
  \includegraphics[width=1.0\linewidth]{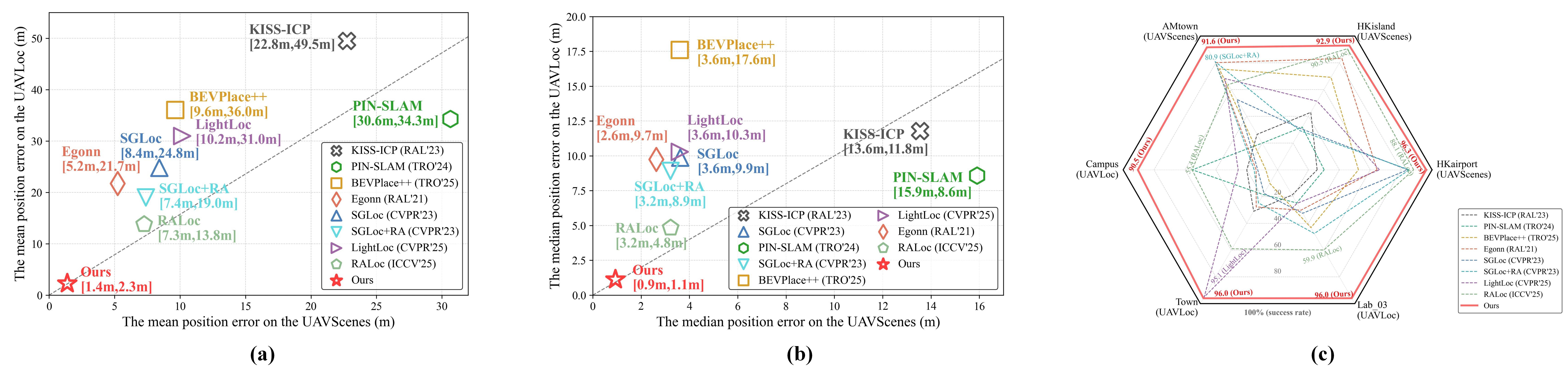}
  \caption{The results of the mean and median position error (m), the success rate (Thresholds=5m,5\degree) on UAVScenes and UAVLoc (ours). Our method achieves the best performance.}
\label{fig1_2}
\end{figure*}

Recently, unmanned aerial vehicles have emerged as key enablers of the low-altitude economy~\cite{add_37, 1_2}, supporting a wide range of applications such as terrain exploration, urban inspection, and emergency response~\cite{1_1,1_2, add_36}. During UAV operations, accurate real-time localization is essential for reliable mission execution. Most existing systems rely on GNSS signals~\cite{add_15} to determine the location of UAVs. However, GNSS signals can be easily obstructed in urban environments or become entirely unavailable. Therefore, it is crucial to develop efficient and robust localization methods for UAVs.

Benefiting from LiDAR’s robustness to environmental interference and illumination changes~\cite{add_9, add_12, nuscenes}, LiDAR-based localization provides a promising alternative to address the aforementioned problem~\cite{add_4, 1_8, add_2, add_3, PiLoT}. Several studies have explored LiDAR odometry~\cite{add_13, pinslam, add_31, dtd} for UAV localization~\cite{1_3,1_4,add_10,add_11}. However, such methods suffer from accumulated drift over time, leading to increasing localization errors and limiting their applicability to long-term UAV navigation in real-world scenarios~\cite{ntu, 5-1}.

With advances in sensor technology, acquiring large-scale, high-precision LiDAR maps has become relatively straightforward. As a result, numerous methods focus on 6-degrees-of-freedom (6-DoF) relocalization~\cite{1_5,spinnet,spinnet-tpami, add_7, 53_casspr}. Despite achieving promising performance, these methods require expensive 3D map storage and incur high communication overhead~\cite{1_6,pinslam,buffer, add_6}, which poses severe challenges for resource-constrained platforms such as UAVs~\cite{1_8}. This motivates the development of relocalization methods that avoid explicit map construction and retrieval. Following prior work~\cite{map_free}, we use the term \textit{map-free} to denote methods that do not maintain explicit geometric maps during inference. Instead, scene priors are implicitly encoded within network parameters through training, enabling pose estimation directly from input LiDAR scans via learned scene representations. In autonomous driving~\cite{add_16,DRIR}, many methods have explored map-free LiDAR relocalization~\cite{RALoc,Lightloc,lisa, add_8, Opal}, which train deep neural networks to memorize scenarios and directly regress the global pose for each input LiDAR scan. Unfortunately, such methods cannot be directly applied to drone platforms. As illustrated in Fig.~\ref{fig_1}, unlike ground vehicles, UAVs operate under (1) more complex flight conditions, (2) more diverse path and altitude options leading to sparse overlap, and (3) greater 6-DoF rotational variations (especially in yaw direction), which differs fundamentally from the ground vehicle scenario.

These challenges require jointly modeling local geometric structures and global invariance, which are not addressed by existing methods. Therefore, we design a scene coordinate regression-based LiDAR relocalization framework for UAVs under the implicit map-free setting. To ensure robustness under UAV conditions, our model satisfies these properties: 1) Rotation invariance in yaw. The model remains unaffected by yaw rotations that commonly occur during UAV flight. 2) Altitude robustness. For LiDAR scans captured at varying altitudes, the model encodes them into similar representations, thereby mitigating the influence of altitude variations. 3) Local descriptiveness. Since UAV trajectories often yield point clouds with partial overlap, the model must exhibit strong discriminative capability for local geometric features to handle complex and irregular flight paths.

Specifically, we propose \method{}, a map-free LiDAR relocalization framework for high-accuracy and robust pose estimation for UAVs. To handle large yaw and altitude variations, we introduce the Locality-Preserving Sliding Window Attention (LoSWAtt) module to extract locally geometric features. Moreover, to encode discriminative local representations, we propose Softmax-free module, which is simple yet remarkably effective. Finally, to evaluate relocalization performance in real-world, we build a LiDAR-equipped drone platform and collect a large-scale outdoor dataset, including irregular and diverse flight paths in four scenes. As shown in Fig.~\ref{fig1_2}, \method{} achieves state-of-the-art performance on two UAV datasets. Overall, our contributions are summarized as follows:

\begin{itemize}
\item As the first systematic study of map-free LiDAR relocalization for UAVs, we analyze the key challenges in this setting, and provide a tailored solution.
\item We propose a new LiDAR relocalization framework that jointly extracts discriminative geometric features and global invariance, which is robust to azimuthal and altitude changes.
\item We collect a large-scale LiDAR localization dataset with diverse scenes and irregular trajectories, providing a realistic and challenging benchmark for future research.
\item Extensive experiments demonstrate that our method significantly outperforms existing approaches and offer insights into the design choices behind it.
\end{itemize}

\begin{figure*}
  \centering
  \includegraphics[width=1\linewidth]{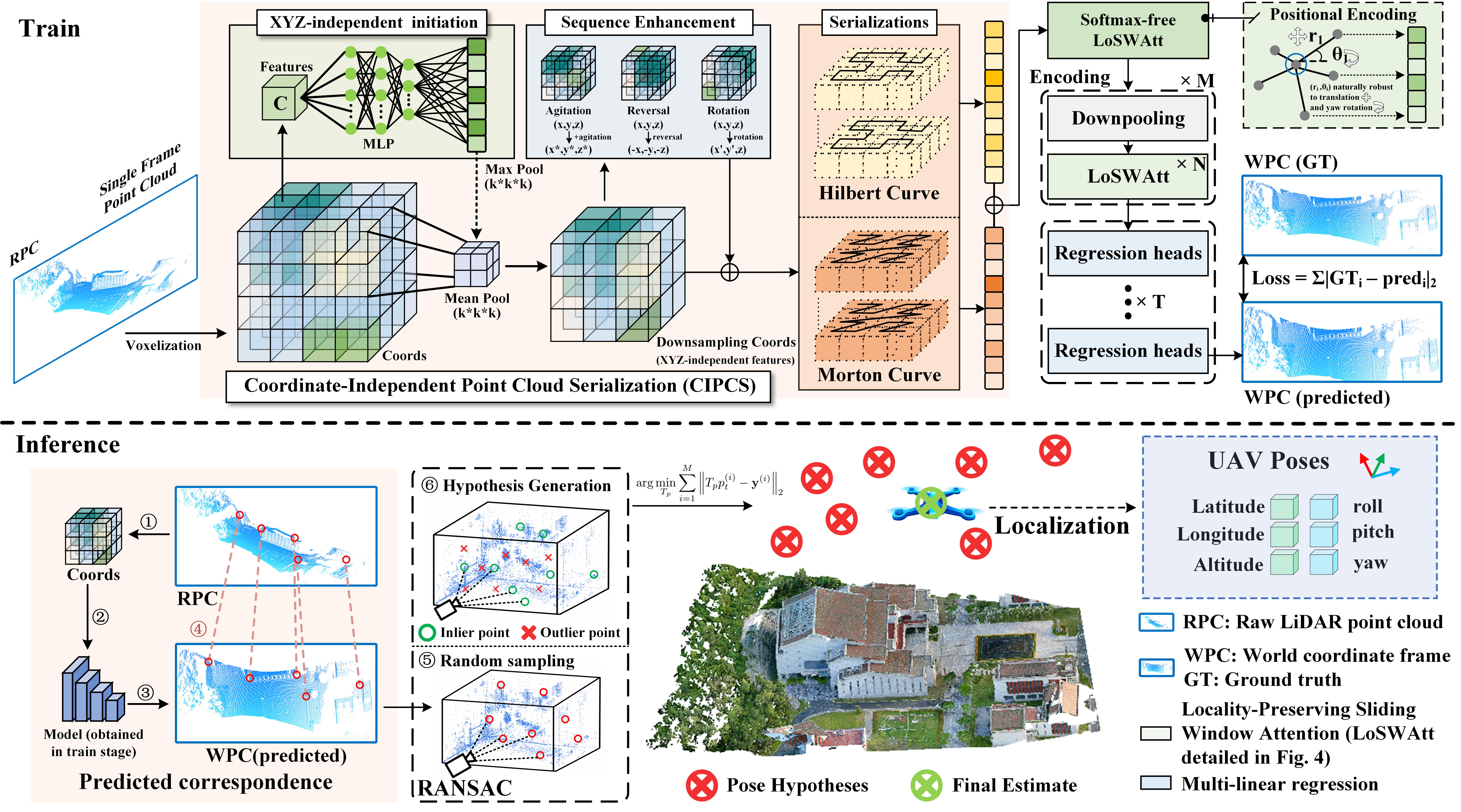}
  \caption{\textbf{The pipeline of \method{},} which consists of two main components:(1) the CIPCS module initializes features to serializations. (2) the LoSWAtt module encodes robust features (detailed in Fig. \ref{fig3}). RANSAC is used for 6-DoF pose estimation in inference stage. Detailed implementation, including the specific values for M, N and T, can be found in the appendix.}
\label{fig2}
\end{figure*}

\section{Related work}
\subsection{Conventional Relocalization}


Conventional relocalization methods \cite{add_23,add_24,add_25} aim to align the query point cloud with a pre-constructed 3D map. These approaches are typically categorized into two main types: retrieval-based methods \cite{29_Minkloc3d, 48_pointnetvlad, 53_casspr} and matching-based methods \cite{42_geometric}. Retrieval-based methods treat relocalization as a task of recognizing a place. To do this, they need a pre-built database of feature descriptors and use similarity queries to find the most similar point cloud in the database \cite{59_std}. On the other hand, matching-based methods retrieve candidate frames based on similarity in a descriptor database or 3D map. Then, they match features of query frame to those of reference map \cite{33_extend}. However, both types inherently rely on pre-built maps, leading to challenges in data storage and communication efficiency.

\subsection{Regression-based Localization}

Recently, deep learning-based relocalization methods \cite{add_21, add_22, add_26, add_27} have achieved remarkable progress by formulating the task as a regression problem, thereby enabling direct end-to-end predicting 6-DoF poses. Specifically, PointLoc \cite{1_9} pioneered a LiDAR-based Absolute Pose Regression (APR) method, realizing map-free relocalization for vehicles. Subsequently, HypLiLoc \cite{hypliloc} leveraged multi-level features to capture more discriminative features. DiffLoc \cite{diffloc} improved accuracy by transforming pose regression process into multi-iterative steps through a diffusion model \cite{diffusion, add_34}.

However, as noted in \cite{SGLoc}, APR methods primarily emphasize high-dimensional pose–feature mappings, which constrains their accuracy—particularly in large-scale outdoor environments \cite{RALoc}. To overcome this, Scene Coordinate Regression (SCR) methods~\cite{add_32, add_33} were introduced, which predicts the world-coordinate position of each point, and subsequently employ RANSAC to estimate the final 6-DoF pose. SGLoc \cite{SGLoc} was the first to integrate the SCR framework into LiDAR-based relocalization. Subsequently, LiSA \cite{lisa} incorporated semantic information to reduce the influence of noisy or disruptive points within the scene, further improving localization accuracy. More recently, RALoc \cite{RALoc} has focused on mitigating rotation-related challenges during pose estimation. These methods work well for vehicle-based LiDAR relocalization, but their performance degrades significantly in UAV scenarios. 
Therefore, we aim to address challenges of UAV relocalization.

\section{Method}



Although SCR-based methods have demonstrated strong performance on ground-vehicle datasets~\cite{SGLoc,RALoc}, as illustrated in Fig.~\ref{fig_1}, UAVs exhibit more complex flight trajectories, larger azimuthal changes, and greater altitude variations than ground vehicles, which significantly degrade the performance of existing localization methods tailored to autonomous driving. To address these challenges, we propose \method{}, which is robust to yaw rotation and height variations, enabling precise relocalization in UAV scenarios. Moreover, we collect UAVLoc, a large-scale UAV LiDAR dataset, to evaluate performance under realistic scenes and alleviate the limitations of existing benchmarks.


\subsection{\method{}}

\noindent\textbf{Framework.} Similar to SCR-based vehicle methods \cite{SGLoc}, we define the overall \method{}~framework as follows. Given a query point cloud $\mathbf{P}_t \in \mathbb{R}^{N \times 3}$, we aim to estimate the UAV’s global 6-DoF pose $\mathbf{T}$. As shown in Fig. \ref{fig2}, \method{} first learns to regress point-wise correspondences from the raw point cloud to predict its world coordinates $\mathbf{Y}$. Then, this mapping is realized by a learnable model $\mathcal{F}$, i.e.
\begin{equation}
\mathbf{Y} = \mathcal{F}(\mathbf{P}_t).
\end{equation}
During inference, we leverage the predicted correspondences ($\mathbf{P}_t ,\mathbf{Y}$) to estimate the pose. We employ RANSAC~\cite{RANSAC, add_35} to robustly sample $M$ candidate correspondence sets and solve for the transformation that minimizes the following energy function:
\begin{equation}
\mathbf{T}^* = \arg\min_{\mathbf{T}} \sum_{i=1}^{M} \left\| \mathbf{T}\mathbf{p}_t^{(i)} - \mathbf{y}^{(i)} \right\|_{2}
\end{equation}
where $\mathbf{T}$ is a $4 \times 4$ matrix representation of the pose, $i$ is the index of the point, $\mathbf{p}_t^{(i)}$ is a point in $\mathbf{P}_t$, and $\mathbf{y}^{(i)}$ is its corresponding predicted global coordinate in $\mathbf{Y}$.

Specifically, as illustrated in Fig.~\ref{fig2}, \method{}~is divided into two main components: (1) Coordinate-Independent Point Cloud Serialization (CIPCS), which initializes point cloud features through a serialization process that is independent of their original coordinates (XYZ). (2) Locality-Preserving Sliding Window Attention (LoSWAtt) module, which employs a sliding-window local attention encoder together with yaw- and altitude-invariant positional encoding, enabling the encoding of features that remain robust to azimuth and altitude variations.

\noindent\textbf{Coordinate-Independent Point Cloud Serialization.} To address yaw and altitude variations in UAV scenarios, we designed a Coordinate-Independent Point Cloud Serialization (CIPCS) method (shown in Fig. \ref{fig2}). Due to the complex flight trajectories of UAVs, the relative XYZ coordinates of the same point cloud often undergo substantial and frequent variations. To address this, we replace the XYZ coordinates with a constant C as the raw feature for each point. The raw features are then expanded  via an MLP (Multi-Layer Perceptron), followed by downsampling of both voxels and features to obtain the final point cloud feature $\mathbf{F}$, i.e.  
$\mathbf{F} = \text{Downsampling}\Big( \big\{ \varphi(\mathbf{C}) \big\}_{i=1}^N \Big),$
where $N$ is the number of points, $i$ is the index of the point, $\varphi$ denotes multi-layer perceptron, and C is 1 in this paper. Although the raw coordinates are replaced with constants, geometric relationships are implicitly preserved through voxelization and serialization, which shifts the burden of representation learning from raw coordinates to structural relationships. After that, we serialize the point cloud voxels using Hilbert and Morton curves which have strong locality-preserving properties~\cite{pointtv3}.

\noindent\textbf{Locality-Preserving Sliding Window Attention.} To initialize point cloud features with invariance to yaw and altitude variations, we propose the Locality-Preserving Sliding Window Attention (LoSWAtt) module, as illustrated in Fig.~\ref{fig3}. Specifically, we compute a local attention mechanism over the point cloud sequence using a sliding window of size $k$. This sliding window design not only enables the model to extract locally invariant geometric features but also effectively reduces the spatial complexity from $O(n^2)$ to $O((2k+1)n)$. For the sliding window at step $i$, we define local input feature $\mathbf{F}^w_i$ as:
\begin{equation}
\mathbf{F}^w_i = \big\{ \mathbf{f}_{i - k}, \mathbf{f}_{i - k + 1}, \dots, \mathbf{f}_{i}, \dots, \mathbf{f}_{i + k} \big\}
\label{eq3}
\end{equation}
Subsequently, we encode the $\mathbf{F}^w_i$ using an attention mechanism with a positional bias that preserves local invariance. The specific process is shown in Eq. \ref{eq4}.
\begin{equation}
\mathbf{f}'_i = \text{Softmax}\Bigg(
\frac{\mathbf{Q}\mathbf{K}^\top}{\sqrt{D}} + \frac{\mathbf{Q}_2 \mathbf{K}_2^\top}{\sqrt{D_2}}
\Bigg) \mathbf{V},
\label{eq4}
\end{equation}
where D and D$_2$ denote the dimensions of Q and Q$_2$, and $\mathbf{Q},\mathbf{K},\mathbf{V}=P_{Q,K,V}(\mathbf{F}^w_i), \mathbf{Q}_2 = \varphi_2(\mathbf{f}_i)$. Simultaneously, we compute the relative position $r$ of each point within the sliding window relative to central point, and use the pitch angle $\theta$, independent of yaw, as $\mathbf{K}_2$, thereby providing encoder with positional encoding that is invariant to both yaw and altitude variations, which is shown in Eq. \ref{eq5}.
\begin{equation}
\mathbf{K}_2 = \varphi_3\left[ (r_{i-k}, \theta_{i-k}), \dots, (r_{i+k}, \theta_{i+k}) \right],
\label{eq5}
\end{equation}
Then, we obtain predicted world coordinates $\mathbf{Y}'$ for each point using multiple linear regression heads.

\begin{figure*}
  \centering
\includegraphics[width=1\linewidth]{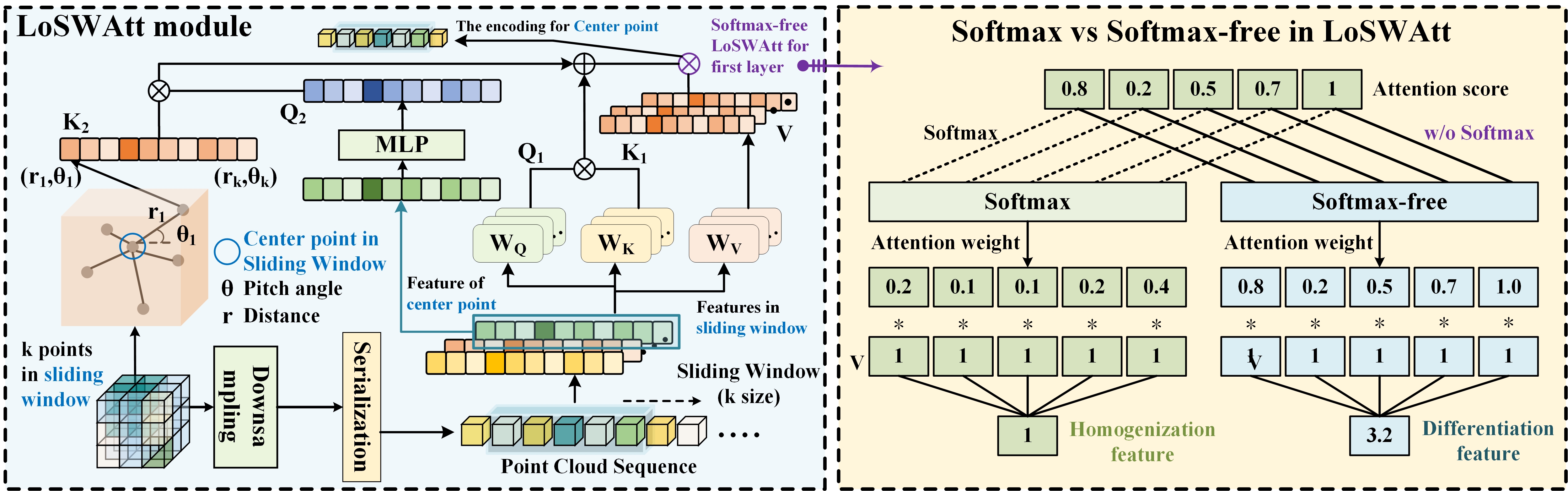}
  \caption{The Locality-Preserving Sliding Window Attention module, designed to encode local geometric features. Notably, LoSWAtt's first layer employs Softmax-free to produce distinct features.}
\label{fig3}
\end{figure*}

\noindent\textbf{Softmax-free LoSWAtt.} Although the CIPCS eliminates the effects of XYZ rotation and translation, the resulting feature points across different point clouds are not distinctive, because the constant C is transformed into the same constant C$_2$ through the $\varphi$. Consequently, these features cannot be directly used in subsequent regression encoding. Therefore, we achieve distinctive point cloud features by designing the first-layer LoSWAtt module \textbf{without Softmax}, which is simple but necessary.

Specifically, as illustrated in Fig.~\ref{fig3} (w/o Softmax), the point cloud feature $\mathbf{F}$ is treated as a constant C$_2$ and thus remains constant after linear projection. Consequently, Q, K, Q$_2$, and V in Eq.~\ref{eq4} all collapse to constant vectors. Although K$_2$ introduces variation into the attention scores, the Softmax normalization suppresses these variations, leading to degraded attention expressiveness and resulting in highly homogeneous features. Therefore, this design avoids normalization-induced feature collapse under homogeneous inputs. This phenomenon is clearly verified by the ablation results presented in Tab.~\ref{table4}, where we observed the performance of \method{}~degrades significantly when the first LoSWAtt module is used with Softmax. More analysis and theoretical proof of Softmax-free LoSWAtt can be found in the appendix.

\noindent\textbf{Loss function.} During training, the point cloud with global poses predicted by the network $\mathcal{F}$ is optimized by $\mathcal{L}_{L1}$. Specifically, we minimize the average L1 distance between the predicted scene coordinates $y_i$, and the ground truth of scene coordinates $y^{*}_i$:
\begin{equation}
\mathcal{L}_{L1} = \frac{1}{N} \sum_{i=1}^{N} \left\| \mathbf{y}_i - \mathbf{y}_i^* \right\|_{1}
\end{equation}

\begin{table*}
  \caption{\textbf{Comparison with other UAV datasets.} 
  ``Single-frame PC Acquire'' represents whether single-frame point cloud data is acquired. ``Vertical Resolution'' denotes the vertical resolution of the LiDAR. ``Points/sec'' refers to the peak Points/sec specified by the manufacturer. ``Multiple Traversals'' represents multiple trajectories in each scene that overlap sufficiently to be used for train and test. ``Various Paths'' represents different flight paths and altitudes  between multiple traversals. ``LiDAR Reloc'' represents whether it can be used for evaluating LiDAR relocalization. ``N/A'' represents this item does not apply. ``{$\checkmark{}\mkern-11mu\raisebox{0.8ex}{$\scriptstyle\smallsetminus$}$}'' represents that NTU VIRAL provides the relative poses derived from IMU. ``--'' represents that it hasn't a real LiDAR. ``$\ast$'' represents solid LiDAR lacks vertical resolution.}
  \vspace{1mm}
  \label{tab:example}
  \centering
  \resizebox{1\linewidth}{!}{
  \begin{tabular}{cc|ccc|ccccccc}
\hline
Dataset    & Year & LiDAR Type & \begin{tabular}[c]{@{}c@{}}LiDAR \\ Threads\end{tabular} & Scenes & \begin{tabular}[c]{@{}c@{}}Single-frame \\ PC Acquire\end{tabular} & \begin{tabular}[c]{@{}c@{}}Vertical \\ Resolution\end{tabular} & \begin{tabular}[c]{@{}c@{}}Points \\ /sec\end{tabular} & \begin{tabular}[c]{@{}c@{}}6-DoF \\ Pose\end{tabular}   &   \begin{tabular}[c]{@{}c@{}}Multiple\\ Traversals\end{tabular} & \begin{tabular}[c]{@{}c@{}}Various\\ Paths \end{tabular} & \begin{tabular}[c]{@{}c@{}}LiDAR \\ Reloc\end{tabular} \\  \hline
Mid-Air\cite{mid-air} & 2019 & no real LiDAR & N/A & Simulation  & \checkmark & \textcolor{red}{--} & \textcolor{red}{--} & \checkmark & \checkmark & \checkmark & \checkmark \\
TartanAir\cite{Tartanair} & 2020 & no real LiDAR & N/A & Simulation & \checkmark & \textcolor{red}{--} & \textcolor{red}{--} & \checkmark & \checkmark & \checkmark & \checkmark \\
University-1652\cite{University-1652} & 2020 & no real LiDAR & N/A & Google Earth & \checkmark & \textcolor{red}{--} & \textcolor{red}{--} & \textcolor{red}x & \textcolor{red}x & \textcolor{red}x & \textcolor{red}x \\
UrbanScene3D\cite{urbanscene3d} & 2022 & no real LiDAR & N/A & Simulation & \checkmark & \textcolor{red}{--} & \textcolor{red}{--} & \checkmark  & \checkmark  & \checkmark  & \checkmark  \\
SynDrone\cite{Syndrone} & 2023 & no real LiDAR & N/A & Simulation & \checkmark & \textcolor{red}{--} & \textcolor{red}{--} & \textcolor{red}x & \textcolor{red}x & \textcolor{red}x & \textcolor{red}x \\ 
UrbanBIS\cite{UrbanBIS} & 2023 & no real LiDAR & N/A & Simulation & \checkmark & \textcolor{red}{--} & \textcolor{red}{--} & \textcolor{red}x & \textcolor{red}x & \textcolor{red}x & \textcolor{red}x \\ \hline

CARPK\cite{CARPK} & 2017 & \textcolor{red}{Camera only} & N/A & Parking Lot & \textcolor{red}x & \textcolor{red}{--} & \textcolor{red}{--} & \textcolor{red}x & \textcolor{red}x & \textcolor{red}x & \textcolor{red}x \\
UAVDT\cite{UAVDT} & 2018 & \textcolor{red}{Camera only} & N/A & Urban Traffic & \textcolor{red}x &\textcolor{red}{--} & \textcolor{red}{--} & \textcolor{red}x & \textcolor{red}x & \textcolor{red}x & \textcolor{red}x \\
VisDrone\cite{VisDrone} & 2018 & \textcolor{red}{Camera only} & N/A & Urban \& Campus & \textcolor{red}x & \textcolor{red}{--} & \textcolor{red}{--} & \textcolor{red}x & \textcolor{red}x & \textcolor{red}x & \textcolor{red}x \\
UAVid\cite{UAVid} & 2020 & \textcolor{red}{Camera only} & N/A & Urban Roads & \textcolor{red}x & \textcolor{red}{--} & \textcolor{red}{--} & \textcolor{red}x & \textcolor{red}x & \textcolor{red}x & \textcolor{red}x \\
FloodNet\cite{Floodnet} & 2021 & \textcolor{red}{Camera only} & N/A & Urban & \textcolor{red}x & \textcolor{red}{--} & \textcolor{red}{--} & \textcolor{red}x & \textcolor{red}x & \textcolor{red}x & \textcolor{red}x \\
CrossLoc\cite{Crossloc} & 2022 & \textcolor{red}{Camera only} & N/A & Urban \& Rural \& Farm \& Nature & \textcolor{red}x & \textcolor{red}{--} & \textcolor{red}{--} & \checkmark & \checkmark & \textcolor{red}x & \checkmark \\
Drone Vehicle\cite{DroneVe} & 2022 & \textcolor{red}{Camera only} & N/A & Urban \& Road & \textcolor{red}x & \textcolor{red}{--} & \textcolor{red}{--} & \textcolor{red}x & \textcolor{red}x & \textcolor{red}x & \textcolor{red}x \\
SUES-200\cite{SUES-200} & 2023 & \textcolor{red}{Camera only} & N/A & Urban \& School \& Lake \& Park & \textcolor{red}x & \textcolor{red}{--} & \textcolor{red}{--} & \textcolor{red}x & \checkmark & \textcolor{red}x & \textcolor{red}x \\
UAV-VisLoc\cite{Uav-visloc} & 2024 & \textcolor{red}{Camera only} & N/A & Urban \& Town \& Farm \& River & \textcolor{red}x & \textcolor{red}{--} & \textcolor{red}{--} & \textcolor{red}x & \checkmark & \checkmark & \textcolor{red}x \\
HazyDet\cite{HazyDet} & 2024 & \textcolor{red}{Camera only} & N/A & Urban & \textcolor{red}x & \textcolor{red}{--} & \textcolor{red}{--} & \textcolor{red}x & \textcolor{red}x & \textcolor{red}x & \textcolor{red}x \\
UAVD4L\cite{Uavd4l} & 2024 & \textcolor{red}{Camera only} & N/A & Urban \& Rural & \textcolor{red}x & \textcolor{red}{--} & \textcolor{red}{--} & \checkmark & \textcolor{red}x & \textcolor{red}x & \checkmark \\ \hline

Hessigheim 3D\cite{Hessigheim} & 2021 & RIEGL VUX-1LR & Solid & Town & \checkmark & $\ast$ & 1.5M & \textcolor{red}x & \textcolor{red}x & \textcolor{red}x & \textcolor{red}x \\
NTU VIRAL\cite{ntu} & 2022 & 2 x Ouster-16 & 16 & Campus \textcolor{red}{(Small Scale)} & \checkmark & \textcolor{red}{2-3\degree} & 1.3M & \textcolor{red}{$\checkmark{}\mkern-11mu\raisebox{0.8ex}{$\scriptstyle\smallsetminus$}$} & \checkmark & \checkmark & \checkmark \\
UrbanScene3D\cite{urbanscene3d} & 2022 & Trimble-X7 & Solid & Urban & \textcolor{red}x  & $\ast$ & 500k & \textcolor{red}x  & \textcolor{red}x  & \textcolor{red}x  & \textcolor{red}x  \\
GraCo\cite{Graco} & 2023 & Velodyne-16 & 16 & Campus & \checkmark & \textcolor{red}{\textasciitilde2\degree} & 600k & \textcolor{red}x & \checkmark & \checkmark & \textcolor{red}x \\
GauU-Scene V2\cite{Gauu-scene} & 2024 & DJI-L1 & Solid & Urban \& Town \& Campus & \textcolor{red}x & $\ast$ & 480k & \checkmark & \checkmark & \checkmark & \textcolor{red}x \\
MUN-FRL\cite{mun} & 2024 & Velodyne-16 & 16 & Urban \& Road \& Airports \& Nature & \checkmark & \textcolor{red}{\textasciitilde2\degree} & 600k & \checkmark & \textcolor{red}x & \textcolor{red}x & \textcolor{red}x \\
FIReStereo\cite{Firestereo} & 2025 & Velodyne-16 & 16 &  Urban \& Forest & \checkmark & \textcolor{red}{\textasciitilde2\degree} & 600k & \textcolor{red}x & \checkmark & \checkmark & \textcolor{red}x \\
MARS-LVIG\cite{MARS-LVIG} & 2024 & DJI-L1 \& Livox-Avia & Solid & Airport \& Island \& Town \& Valley & \checkmark & $\ast$ & 480k \& 720k & \textcolor{red}x & \checkmark & \textcolor{red}x & \textcolor{red}x \\
UAVScenes\cite{UAVScenes} & 2025 & Livox-Avia & Solid & Airport \& Island \& Town \& Valley & \checkmark & $\ast$ & 720k & \checkmark & \checkmark & \textcolor{red}x & \checkmark \\ \hline

\textbf{UAVLoc (ours)}   & 2026  & 1 x Ous.-128 & 128 & Lab Park \& Campus \& Town \& Road & $\pmb{\checkmark}$ & \textbf{0.35}\degree & \textbf{5.2M} & $\pmb{\checkmark}$   & $\pmb{\checkmark}$ & $\pmb{\checkmark}$ & $\pmb{\checkmark}$ \\ \hline
\end{tabular}}
\label{table1}
\vspace{-1mm}
\end{table*}

\subsection{Dataset}

Most existing UAV datasets are limited to evaluate LiDAR relocalization in real-world, due to the lack of raw LiDAR scans, accurate 6-DoF poses, or multiple traversals. 
To address it, we construct UAVLoc, a LiDAR dataset for UAVs with irregular flight trajectories and extensive rotational variations in 4 scenes, as shown in Fig.~\ref{fig4}(a). In this section, we focus on the unique characteristics that distinguish UAVLoc from existing datasets; more details of UAVLoc, please refer to the Appendix.

\noindent\textbf{Comparison with UAV datasets.} Tab. \ref{table1} provides a comprehensive comparison between UAVLoc and other UAV datasets. Our goal is to collect a large-scale outdoor dataset with irregular flight paths, including variations in rotation, position, and altitude, to validate UAV relocalization algorithms. Overall, our dataset highlights the following key characteristics: 1) Extensive and large-scale scenes with high-quality point clouds. 2) Multiple traversals with \textbf{irregular trajectories} and \textbf{extensive variations in rotation and altitude} within every scene, providing realistic UAV flight conditions.

Compared to datasets~\cite{mid-air,Tartanair,University-1652,urbanscene3d,Syndrone,UrbanBIS} generated in simulators (such as AirSim), UAVLoc can reflect real-world operating conditions of UAVs.
Some UAV datasets~\cite{CARPK,UAVDT,VisDrone,UAVid,Floodnet,Crossloc,DroneVe,SUES-200,Uav-visloc,HazyDet,Uavd4l} are collected using only cameras. In contrast, LiDAR provides more reliable data in low-visibility conditions, such as at night.

Compared with other LiDAR UAV datasets, UAVLoc is the unique one that simultaneously include large-scale, irregular trajectories and various altitudes for UAV relocalization. Specifically, UrbanScene3D~\cite{urbanscene3d} is listed twice to distinguish simulation-based version and real-world version; however, it does not provide real-world LiDAR data in both version; GauU-Scene V2 \cite{Gauu-scene} uses the DJI-L1 LiDAR, which prevents the acquisition of per-frame point clouds. As point clouds cannot be acquired, they cannot be used for relocalization. The datasets~\cite{Hessigheim,Graco,Firestereo,MARS-LVIG} did not provide ground truth for 6-DoF poses; consequently, they also cannot be used to evaluate relocalization. In contrast, MUN-FRL~\cite{mun} offered 6-DoF poses, but it is equipped only with a short-range and low-resolution LiDAR on the high altitude, resulting in substantial missing ground observations. Moreover, it exhibits minimal overlap between multiple traversals, which makes it less suitable for evaluating relocalization. Instead, our dataset offers advantages in terms of vertical angular resolution and peak point rate, which directly affect the geometric observability and structural fidelity of point clouds for relocalization. NTU VIRAL~\cite{ntu} is a dataset designed for UAV SLAM, including LiDAR data and 6-DoF poses. However, it is collected in indoor and small-scale outdoor environments and only provides relative poses derived from IMU, limiting its application for large-scale scene understanding and relocalization. UAVScenes~\cite{UAVScenes} was the latest dataset and was suitable for evaluating large-scale outdoor UAV LiDAR relocalization. Unlike UAVScenes, which uses fixed altitudes and predefined flight paths, our dataset provides more complex flight data that more closely reflects the real-world operational conditions of UAVs. In summary, our dataset shows uniqueness and novelty in evaluating real-world UAV relocalization with challenging and realistic UAV flight data.

\begin{figure*}
  \centering
  \includegraphics[width=1\linewidth]{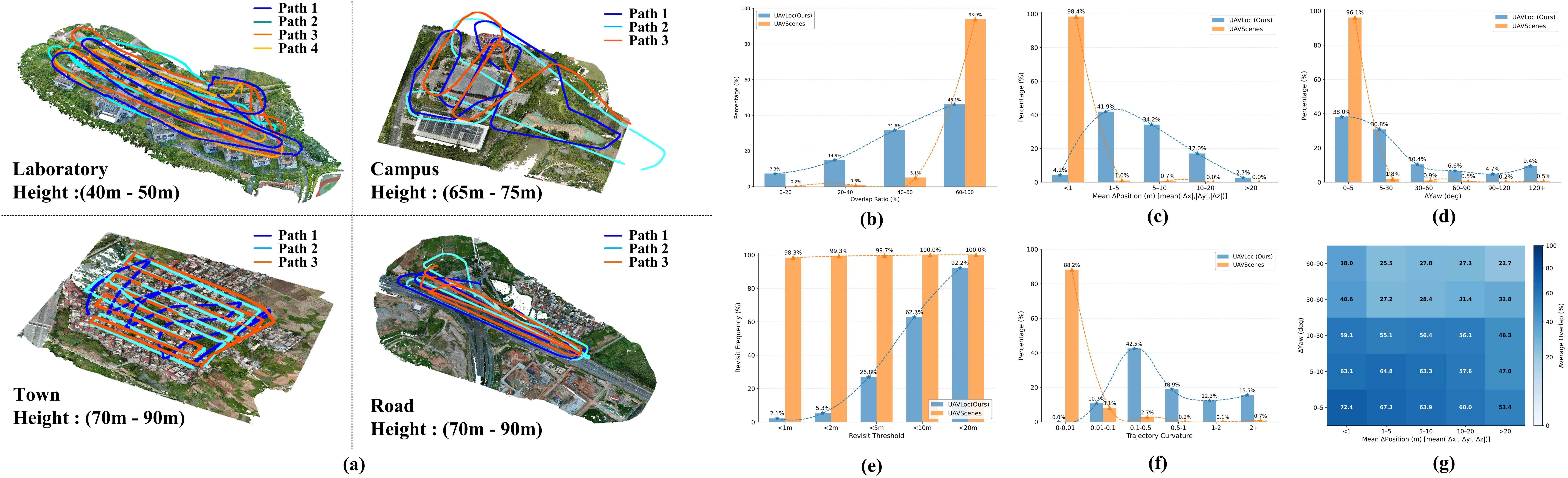}
  \caption{Fig.(a) shows irregular trajectories of our dataset. Fig.(b)–(f) show our dataset exhibits richer distributions of overlap, curvature, positional \& yaw offset, and revisit frequency compared to UAVScenes. Fig.(g) shows the impact of positional and yaw offset on overlap.}
\label{fig4}
\end{figure*}

\noindent\textbf{Quantitative comparison.} To provide further evidence of our superiority, we have quantitatively compared UAVLoc with the latest UAVScenes dataset in Fig.\ref{fig4}(b)-(f). As shown in Fig.\ref{fig4}(b)-(d), we present their distributions of the overlap ratio, positional offset($\Delta$Position), and yaw offset($\Delta$yaw). Specifically, we calculate the overlap ratio by projecting two point clouds (one in test set and its nearest based on ground-truth in train set) into 1m voxel grids and measuring their set intersection, i.e., $\mathrm{Overlap}(P_q, P_r) = \frac{| \mathcal{V}(P_q) \cap \mathcal{V}(P_r) |}{| \mathcal{V}(P_q) |}$; and calculate the $\Delta$Position and $\Delta$yaw also by computing the mean(x, y, z) and yaw offset for these nearest paired point clouds. Fig.\ref{fig4}(b)-(d) show that UAVLoc exhibits richer and more challenging distributions of overlap ratio, $\Delta$Position and $\Delta$yaw, compared to UAVScenes. This is due to the irregular flight paths in UAVLoc. Moreover, as shown in Fig.\ref{fig4}(g), there is a marked correlation between the increase in $\Delta$Position and $\Delta$yaw and the decrease in the overlap ratio. Similarly, the irregular paths led to a decrease in our revisit rate (i.e. the proportion of frames in test set successfully matched to frame in train set within thresholds). As shown in Fig.\ref{fig4}(e), almost all frames in UAVScenes have a revisited point within a 1-metre radius, showing that the flight path is almost entirely fixed; however, this does not reflect the complex flight conditions of UAVs in real world. Finally, we calculate curvature distributions for two datasets by estimating discrete trajectory curvature at each point using three consecutive positions, i.e., $\kappa_i = \frac{2 \left| (\mathbf{p}_i - \mathbf{p}_{i-1}) \times (\mathbf{p}_{i+1} - \mathbf{p}_{i}) \right|}{\left| \mathbf{p}_{i} - \mathbf{p}_{i-1} \right| \left| \mathbf{p}_{i+1} - \mathbf{p}_{i} \right| \left| \mathbf{p}_{i+1} - \mathbf{p}_{i-1} \right|}$, where $\kappa_i$ denotes the curvature and $\mathbf{p}_{i}$ denotes the 3D position at timestep $i$. As shown in Fig.~\ref{fig4}(f), Our dataset exhibits a richer distribution of curvature, which demonstrates more frequent rotational changes. Overall, quantitative comparisons demonstrate the challenging and diverse attributes of our dataset, making it valuable for evaluating real-world UAV deployment scenarios.

More details about UAVLoc, including Sensors Setup, Ground Truth, etc., please see the appendix.
\section{Experiment}

\begin{table*}
  \caption{\textbf{Results on UAVScenes dataset.} We report the Recall@1 under a threshold of 10m, the success rate (SR) under a threshold of 10m/10\degree and the translation and rotation errors [m/\degree] while success (RTE/RRE), as well as mean and median localization errors [m/\degree] on UAVScenes. `--' signifies the failure of the algorithm’s execution, `N/A' signifies HOTForLoc doesn’t explicitly estimate pose and cannot calculate these metrics.}
  \label{tab:example}
  \centering
  \resizebox{1\linewidth}{!}{
\begin{tabular}{lccccc||ccccc}
\hline
\multirow{3}{*}{Method} & \multicolumn{5}{c||}{Seq: HKairport\_3}                                                                                                                                                                       & \multicolumn{5}{c}{Seq: HKisland\_3}                                                                                                                                                                         \\ \cline{2-11} 
                        & \multicolumn{1}{c}{Thr.=5m} 
& \multicolumn{2}{c}{Thr.=5m,5\degree} 
& \multirow{2}{*}{Mean Error} 
& \multirow{2}{*}{Median Error}
& \multicolumn{1}{c}{Thr.=5m} 
& \multicolumn{2}{c}{Thr.=5m,5\degree} 
& \multirow{2}{*}{Mean Error} 
& \multirow{2}{*}{Median Error} \\ \cline{2-4} \cline{7-9}
                        & Recall@1        & SR        & RTE,RRE       &                                                                                    &                                                         & Recall@1        & SR        & RTE,RRE       &                                                                                    &                                                         \\ \hline

Kiss-ICP (RAL'23)
& 36.88\% & 25.60\% & \colorbox{second}{1.17m},1.88\degree & 29.21m,22.26\degree & 11.60m,7.39\degree
& 45.80\% & 42.96\% & 1.87m,1.80\degree & 14.76m,15.98\degree & 11.07m,6.43\degree \\
PIN-SLAM (TRO'24) 
& 30.16\% & 30.16\% & \colorbox{third}{1.55m},1.96\degree & 38.45m,27.10\degree & 15.70m,12.92\degree
& 32.43\% & 31.97\% & \colorbox{second}{1.47m},\colorbox{third}{1.72\degree} & 27.40m,19.24\degree & 12.89m,10.48\degree \\
HOTFLoc (CVPR'25) & 85.19\% & N/A & N/A & N/A & N/A & \cellcolor{third}90.29\% & N/A & N/A & N/A & N/A \\
BEVPlace++ (TRO'25)   & 58.79\% & 52.34\% & 4.16m,3.99\degree & 6.90m,5.37\degree & 5.17m,3.54\degree & 74.23\% & 69.17\% & 2.33m,2.46\degree & 5.19m,3.39\degree & 2.89m,2.90\degree \\
Egonn (RAL'21)  & 75.64\% & 64.39\% & 3.29m,3.01\degree & 4.53m,4.19\degree & 3.51m,2.92\degree & 89.03\% & \cellcolor{third}83.27\% & 1.98m,2.25\degree & \colorbox{third}{3.78m},4.24\degree & \colorbox{third}{1.84m},2.05\degree \\
SGLoc (CVPR'23)  &     \cellcolor{third}90.84\%   &     84.19\%          &   2.00m,\colorbox{third}{1.68\degree}  & \colorbox{third}{2.68m},2.57\degree     &        \colorbox{third}{2.05m},1.67\degree    & 33.23\%   &     30.69\%  &   3.48m,1.96\degree  & 6.83m,\colorbox{third}{3.03\degree}     &    3.44m,\colorbox{third}{1.79\degree}  \\
SGLoc+RA (CVPR'23)                &    88.14\% &  \cellcolor{third}85.24\%  &  2.11m,\colorbox{second}{1.67\degree}  & 3.69m,\colorbox{third}{2.30\degree}      &      3.38m,\colorbox{third}{1.62\degree}   &     33.16\%    &   29.19\%    &       3.59m,\colorbox{third}{1.72\degree} & 6.85m,3.16\degree  &   3.48m,1.95\degree     \\
LightLoc (CVPR'25)                &  77.70\% & 65.96\% & 2.55m,2.21\degree & 3.88m,4.13\degree  & 3.17m,2.79\degree       & 60.55\% & 51.25\% & 3.20m,2.16\degree & 6.95m,5.80\degree & 4.27m,3.14\degree  \\
RALoc (ICCV'25)                   &     \cellcolor{second}93.45\%   &     \cellcolor{second}88.09\%    &   1.68m,2.03\degree  & \cellcolor{second}{1.97m,2.24\degree} &    \cellcolor{second}{1.42m,1.56\degree}         &    \cellcolor{second}97.36\%    &           \cellcolor{second}90.54\%   &     \colorbox{third}{1.57m},\colorbox{second}{1.65\degree}  & \cellcolor{second}{1.81m,2.24\degree} & \cellcolor{second}{1.36m,1.46\degree}  \\
\textbf{\method{} (Ours)} & \cellcolor{best}\textbf{99.50\%} & \cellcolor{best}\textbf{96.26\%} & \cellcolor{best}\textbf{1.01m,1.23\degree} & \cellcolor{best}\textbf{1.06m,1.46\degree} & \cellcolor{best}\textbf{0.85m,0.95\degree}
& \cellcolor{best}\textbf{99.53\%} & \cellcolor{best}\textbf{92.91\%} & \cellcolor{best}\textbf{0.96m,1.31\degree} & \cellcolor{best}\textbf{1.03m,1.85\degree} & \cellcolor{best}\textbf{0.79m,1.00\degree} \\ \hline \hline

\multirow{3}{*}{Method}  & \multicolumn{5}{c||}{Seq: AMtown\_3}                                                                                                                                                                          & \multicolumn{5}{c}{Seq: AMvalley\_3}                                                                                                                                                                         \\ \cline{2-11} 
                       & \multicolumn{1}{c}{Thr.=5m} 
& \multicolumn{2}{c}{Thr.=5m,5\degree} 
& \multirow{2}{*}{Mean Error} 
& \multirow{2}{*}{Median Error}
& \multicolumn{1}{c}{Thr.=5m} 
& \multicolumn{2}{c}{Thr.=5m,5\degree} 
& \multirow{2}{*}{Mean Error} 
& \multirow{2}{*}{Median Error} \\ \cline{2-4} \cline{7-9}
                        & Recall@1        & SR        & RTE,RRE       &                                                                                    &                                                         & Recall@1        & SR        & RTE,RRE       &                                                                                    &                                                         \\ \cline{2-11} 

Kiss-ICP (RAL'23)
& 27.22\% & 26.54\% & \cellcolor{second}{1.50m,1.27\degree} & 24.29m,6.15\degree & 17.98m,4.22\degree
& -- & -- & -- & -- & -- \\
PIN-SLAM (TRO'24) 
& 20.13\% & 19.97\% & 2.58m,1.46\degree & 26.11m,5.48\degree & 19.14m,4.27\degree
& -- & -- & -- & -- & -- \\
HOTFLoc (CVPR'25) & \cellcolor{second}89.76\% & N/A & N/A & N/A & N/A & 76.84\% & N/A & N/A & N/A & N/A \\
BEVPlace++ (TRO'25) & 83.17\% & 75.43\% & 2.15m,2.23\degree & 10.76m,6.77\degree  & \colorbox{third}{2.24m},2.55\degree & 50.41\% & 45.10\% & 2.80m,2.93\degree & 15.71m,7.92\degree & 4.11m,3.12\degree \\
Egonn (RAL'21) & \cellcolor{third}85.93\% & \cellcolor{third}80.11\% & \colorbox{third}{1.91m},1.98\degree & \cellcolor{second}{4.31m,4.20\degree} & \cellcolor{second}{2.17m,2.12\degree} & \cellcolor{second}82.98\% & \cellcolor{third}76.14\% & \colorbox{second}{2.10m},\colorbox{third}{2.08\degree} & \cellcolor{second}{8.33m,4.57\degree} &  \colorbox{second}{3.01m},\colorbox{third}{2.09\degree} \\
SGLoc (CVPR'23) &  65.94\%   &    52.58\%  &    2.97m,2.16\degree   & 12.16m,7.56\degree       &       5.77m,2.66\degree                                                  &    67.64\%    &       63.81\%   &   2.59m,\colorbox{second}{1.72\degree}  & 12.01m,\colorbox{third}{5.00\degree}         &             3.22m,\colorbox{second}{2.00\degree}                                            \\
SGLoc+RA (CVPR'23)  &  82.19\% &  \cellcolor{second}80.88\% &  2.07m,\colorbox{third}{1.86\degree} & \cellcolor{third}{9.43m,5.24\degree}  & 2.82m,\colorbox{third}{2.24\degree} &    \cellcolor{third}82.66\%   & \cellcolor{second}79.23\%  &   \colorbox{third}{2.11m},2.26\degree & \colorbox{third}{9.60m},7.75\degree   &  \colorbox{third}{3.19m},3.01\degree \\
LightLoc (CVPR'25)  & 81.26\% & 68.19\% & 2.02m,2.06\degree & 11.48m,7.65\degree & 2.33m,2.59\degree & 53.86\% & 47.51\% & 2.85m,2.17\degree & 18.31m,8.39\degree  &  4.64m,3.44\degree  \\
RALoc (ICCV'25)  &    75.17\%    &      63.78\%   &    2.47m,2.05\degree    & 11.93m,7.11\degree         &               3.02m,2.65\degree           &    35.15\%    &         20.31\%     &   2.86m,2.80\degree      & 13.32m,5.37\degree      &       7.09m,4.89\degree        \\
\textbf{\method{} (Ours)} & \cellcolor{best}\textbf{98.32\%} & \cellcolor{best}\textbf{91.57\%} & \cellcolor{best}\textbf{0.93m,1.19\degree} & \cellcolor{best}\textbf{1.34m,2.68\degree} & \cellcolor{best}\textbf{0.84m,0.90\degree}
& \cellcolor{best}\textbf{91.70\%} & \cellcolor{best}\textbf{88.82\%} & \cellcolor{best}\textbf{1.43m,1.29\degree} & \cellcolor{best}\textbf{2.14m,2.15\degree} & \cellcolor{best}\textbf{1.27m,1.10\degree} \\ \hline
\end{tabular}}
\label{table2}
\end{table*}

\subsection{Setup}

\noindent\textbf{Datasets and metrics.} We report comprehensive evaluation metrics on UAVScenes~\cite{UAVScenes} and our UAVLoc, including retrieval performance measured by Recall@K, success rate under predefined error thresholds, translation and rotation errors upon successful retrieval (RTE/RRE), as well as the median and mean position and orientation errors. 
For UAVScenes, we use sequence 01 and 02 of each scene for training and 03 for testing. For UAVLoc, we also use sequence 01 and 02 of each scene for training, and other sequences for testing. All method are implemented by Pytorch on a server equipped with an Intel Xeon(R) Gold 6342 CPU, and one NVIDIA RTX 5090 GPUs.

\noindent\textbf{Baseline.}
As this work represents the first study on LiDAR-based map-free UAV relocalization, we select several representative baseline methods to compare, which are among the most advanced approaches in vehicle-LiDAR relocalization. Specifically, for SCR-based methods, we consider SGLoc \cite{SGLoc}, LightLoc \cite{Lightloc}, and RALoc \cite{RALoc} as baseline methods. For SGLoc, we additionally report results of SGLoc with rotation augmentation, denoted as SGLoc+RA, what other methods applied. For conventional-based methods, we consider HOTFormerLoc~\cite{HOTFormerLoc}(Marked as HOTForLoc), BEVPlace++~\cite{bevplace++} and Egonn~\cite{egonn}, where HOTFormerLoc focuses on retrieval and doesn't explicitly estimate poses; therefore, we only report its Recall@1. For odometry and SLAM-based methods. although our primary focus is on global relocalization rather than trajectory estimation, we include results from Kiss-ICP~\cite{kiss_icp} and PinSLAM~\cite{pinslam} as additional references. These methods provide insight into drift accumulation, proving the necessity of global relocalization in our setting. On both UAVScenes and UAVLoc, as some trajectories and their raw point clouds provided minor gaps and missing data, we have run these two methods on them in segments to ensure a fair comparison. All baselines are carefully tuned following official implementations and evaluated under same settings.

\noindent\textbf{4.2 Results}

\begin{table*}
\caption{\textbf{Results on UAVLoc datasets.} We also report the comprehensive metrics on UAVLoc.}
\centering
\resizebox{1\linewidth}{!}{
\begin{tabular}{cccccc||ccccc}
\hline
\multirow{3}{*}{Method} 
& \multicolumn{5}{c||}{Seq: Laboratory\_3} 
& \multicolumn{5}{c}{Seq: Laboratory\_4} \\ \cline{2-11}

& \multicolumn{1}{c}{Thr.=5m\degree} 
& \multicolumn{2}{c}{Thr.=5m,5\degree} 
& \multirow{2}{*}{Mean Error} 
& \multirow{2}{*}{Median Error}
& \multicolumn{1}{c}{Thr.=5m\degree} 
& \multicolumn{2}{c}{Thr.=5m,5\degree} 
& \multirow{2}{*}{Mean Error} 
& \multirow{2}{*}{Median Error} \\ \cline{2-4} \cline{7-9}

& Recall@1 & SR & RTE,RRE 
& & 
& Recall@1 & SR & RTE,RRE 
& & \\ \hline

Kiss-ICP (RAL'23)
& 19.52\% & 18.82\% & \cellcolor{third}{1.99m,1.26\degree} & 26.48m,\colorbox{second}{6.66\degree} & 24.20m,6.42\degree 
& \cellcolor{second}64.21\% & \cellcolor{second}64.21\% & 2.55m,\colorbox{second}{1.26\degree} & \cellcolor{second}{5.27m,1.97\degree} & \cellcolor{second}{3.35m,1.63\degree} \\

PIN-SLAM (TRO'24)
& 25.40\% & 24.72\% & \cellcolor{second}{1.73m,1.02\degree} & 21.65m,\colorbox{second}{6.87\degree} & 10.22m,\colorbox{third}{3.25\degree} 
& 19.90\% & 19.90\% & 4.23m,1.97\degree & 13.73m,\colorbox{third}{8.09\degree} & 10.75m,4.31\degree \\

HOTFLoc (CVPR'25)
& 45.68\% & N/A & N/A & N/A & N/A 
& \cellcolor{third}58.42\% & N/A & N/A & N/A & N/A  \\

BEVPlace++ (TRO'25)
& 49.76\% & 43.20\% & 2.88m,2.65\degree & 14.32m,12.11\degree & \colorbox{third}{6.05m},4.98\degree
& 40.21\% & 36.18\% & 3.17m,2.54\degree & 24.85m,14.21\degree & 6.91m,3.03\degree \\

Egonn (RAL'21) 
& 40.01\% & 30.18\% & 3.59m,3.21\degree & 19.08m,18.24\degree & 6.98m,5.44\degree
& 37.91\% & 25.10\% & 3.67m,3.22\degree & 19.51m,18.36\degree & 7.99m,4.31\degree  \\

SGLoc (CVPR'23)
& 37.13\% & 32.47\% & 2.80m,2.06\degree & 20.08m,15.16\degree & 7.53m,4.33\degree
& 46.17\% & 44.92\% & 2.87m,1.90\degree & 23.46m,12.99\degree & 5.61m,\colorbox{third}{2.53\degree} \\

SGLoc+RA (CVPR'23)
& \cellcolor{third}51.10\% & \cellcolor{third}47.59\% & 2.99m,2.77\degree & \colorbox{third}{13.66m},11.54\degree & 6.20m,5.32\degree
& 41.61\% & 33.33\% & 3.00m,2.72\degree & 11.46m,12.30\degree & 5.90m,3.96\degree \\

LightLoc (CVPR'25)
& 28.55\% & 25.09\% & 2.66m,2.02\degree & 56.07m,33.37\degree & 15.97m,10.01\degree
& 35.40\% & 32.98\% & \colorbox{third}{2.43m},1.87\degree & 37.63m,20.60\degree & 8.82m,4.24\degree \\

RALoc (ICCV'25) 
& \cellcolor{second}65.18\% & \cellcolor{second}59.90\% & 2.07m,1.74\degree & \colorbox{second}{9.68m},7.33\degree & \cellcolor{second}{3.38m,2.71\degree}
& 54.78\% & \cellcolor{third}51.48\% & \colorbox{second}{2.14m},\colorbox{third}{1.66\degree} & \colorbox{third}{10.78m},9.67\degree & \colorbox{third}{4.58m},2.93\degree \\

\textbf{\method{} (Ours)} 
& \cellcolor{best}\textbf{97.54\%} & \cellcolor{best}\textbf{95.98\%} & \cellcolor{best}\textbf{1.12m,0.99\degree} & \cellcolor{best}\textbf{1.34m,1.13\degree} & \cellcolor{best}\textbf{0.90m,0.83\degree}
& \cellcolor{best}\textbf{89.79\%} & \cellcolor{best}\textbf{87.34\%} & \cellcolor{best}\textbf{1.13m,1.08\degree} & \cellcolor{best}\textbf{2.18m,1.89\degree} & \cellcolor{best}\textbf{0.96m,0.87\degree} \\

\hline \hline

\multirow{3}{*}{Method} 
& \multicolumn{5}{c||}{Seq: Campus\_3} 
& \multicolumn{5}{c}{Seq: Town\_3} \\ \cline{2-11}

& \multicolumn{1}{c}{Thr.=5m\degree} 
& \multicolumn{2}{c}{Thr.=5m,5\degree} 
& \multirow{2}{*}{Mean Error} 
& \multirow{2}{*}{Median Error}
& \multicolumn{1}{c}{Thr.=5m\degree} 
& \multicolumn{2}{c}{Thr.=5m,5\degree} 
& \multirow{2}{*}{Mean Error} 
& \multirow{2}{*}{Median Error} \\ \cline{2-4} \cline{7-9}

& Recall@1 & SR & RTE,RRE 
& & 
& Recall@1 & SR & RTE,RRE 
& & \\ \hline

Kiss-ICP (RAL'23)
& 20.17\% & 19.96\% & 2.89m,\colorbox{third}{1.96\degree} & 89.79m,\colorbox{third}{10.25\degree} & 10.64m,4.47\degree
& 31.62\% & 30.98\% & \cellcolor{third}{2.71m,1.70\degree} & 76.31m,9.91\degree & 8.93m,3.05\degree \\

PIN-SLAM (TRO'24)
& \cellcolor{third}55.49\% & \cellcolor{third}55.20\% & \cellcolor{second}{1.61m,1.86\degree} & \colorbox{third}{17.13m},\colorbox{second}{9.96\degree} & \cellcolor{second}{2.56m,1.91\degree}
& 18.77\% & 18.49\% & 3.65m,2.14\degree & 84.55m,12.62\degree & 10.79m,3.92\degree \\

HOTFLoc (CVPR'25)
& 17.28\% & N/A & N/A & N/A & N/A 
& 40.33\% & N/A & N/A & N/A & N/A  \\

BEVPlace++ (TRO'25)
& 11.29\% & 7.62\% & 4.17m,3.01\degree & 51.83m,35.77\degree & 30.64m,34.15\degree
& 17.25\% & 10.12\% & 4.00m,2.76\degree & 43.54m,11.50\degree & 20.23m,25.14\degree \\

Egonn (RAL'21) 
& 17.85\% & 13.88\% & 3.08m,2.81\degree & 44.84m,30.12\degree & 25.18m,29.30\degree
& 35.46\% & 27.73\% & 2.97m,1.99\degree & 11.78m,8.75\degree & 4.88m,3.70\degree \\

SGLoc (CVPR'23)
& 17.10\% & 12.86\% & 3.11m,2.51\degree & 37.63m,28.27\degree & 15.51m,11.20\degree
& 25.62\% & 17.59\% & 3.87m,2.43\degree & 15.53m,9.99\degree & 9.43m,5.24\degree \\

SGLoc+RA (CVPR'23)
& 24.43\% & 15.16\% & 2.98m,2.41\degree & 30.23m,28.95\degree & 14.72m,12.00\degree
& 29.64\% & 24.77\% & 3.11m,2.53\degree & 15.52m,9.98\degree & 7.84m,4.12\degree \\

LightLoc (CVPR'25)
& 30.28\% & 25.54\% & 4.05m,3.88\degree & 42.41m,12.08\degree & 20.78m,10.66\degree
& \cellcolor{second}97.45\% & \cellcolor{second}95.09\% & \cellcolor{second}{1.97m,1.37\degree} & \cellcolor{second}{2.48m,1.62\degree} & \cellcolor{second}{1.82m,1.20\degree} \\

RALoc (ICCV'25) 
& \cellcolor{second}65.05\% & \cellcolor{second}55.37\% & \colorbox{third}{2.15m},2.12\degree & \colorbox{second}{12.83m},11.87\degree &  \cellcolor{third}{3.83m,3.41\degree} 
& \cellcolor{third}64.16\% & \cellcolor{third}58.98\% & 2.86m,1.82\degree & \cellcolor{third}{9.52m,3.00\degree} & \cellcolor{third}{2.93m,2.49\degree} \\

\textbf{\method{} (Ours)} 
& \cellcolor{best}\textbf{95.04\%} & \cellcolor{best}\textbf{90.48\%} & \cellcolor{best}\textbf{1.30m,1.44\degree} & \cellcolor{best}\textbf{2.48m,2.75\degree} & \cellcolor{best}\textbf{1.12m,1.20\degree}
& \cellcolor{best}\textbf{97.47\%} & \cellcolor{best}\textbf{96.01\%} & \cellcolor{best}\textbf{1.27m,1.19\degree} & \cellcolor{best}\textbf{1.52m,1.48\degree} & \cellcolor{best}\textbf{1.09m,1.02\degree} \\

\hline \hline

\multirow{3}{*}{Method} 
& \multicolumn{5}{c||}{Seq: Road\_3} 
& \multicolumn{5}{c}{Average} \\ \cline{2-11}

& \multicolumn{1}{c}{Thr.=5m\degree} 
& \multicolumn{2}{c}{Thr.=5m,5\degree} 
& \multirow{2}{*}{Mean Error} 
& \multirow{2}{*}{Median Error}
& \multicolumn{1}{c}{Thr.=5m\degree} 
& \multicolumn{2}{c}{Thr.=5m,5\degree} 
& \multirow{2}{*}{Mean Error} 
& \multirow{2}{*}{Median Error} \\ \cline{2-4} \cline{7-9}

& Recall@1 & SR & RTE,RRE 
& & 
& Recall@1 & SR & RTE,RRE 
& & \\ \hline

Kiss-ICP (RAL'23)
& -- & -- & -- & -- & --
& 33.88\% & 33.49\% & \cellcolor{second}{2.54m,1.55\degree} & 49.46m,\colorbox{second}{7.20\degree} & 11.78m,3.89\degree \\

PIN-SLAM (TRO'24)
& -- & -- & -- & -- & --
& 29.89\% & 29.58\% & 2.81m,\colorbox{third}{1.75\degree} & 34.27m,\colorbox{third}{9.39\degree} & \colorbox{third}{8.58m},\colorbox{second}{3.35\degree} \\

HOTFLoc (CVPR'25)
& \cellcolor{second}61.47\% & N/A & N/A & N/A & N/A 
& 44.64\% & N/A & N/A & N/A & N/A  \\

BEVPlace++ (TRO'25)
& 16.38\% & 9.92\% & 4.09m,2.69\degree & 45.68m,11.92\degree & 24.15m,26.13\degree
& 26.98\% & 21.41\% & 3.66m,2.73\degree & 36.04m,17.10\degree & 17.60m,18.69\degree \\

Egonn (RAL'21)
& 56.87\% & \cellcolor{third}50.36\% & \colorbox{second}{2.43m},\colorbox{third}{1.89\degree} & \cellcolor{second}13.33m,10.57\degree & \cellcolor{second}{3.67m,3.29\degree}
& 37.24\% & 28.25\%  & 3.15m,2.62\degree  & 21.71m,17.21\degree & 9.74m,9.21\degree \\

SGLoc (CVPR'23)
& 30.45\% & 28.87\%  & 3.97m,4.01\degree  & 27.17m,\colorbox{third}{11.10\degree} & 11.25m,6.13\degree
& 31.29\% & 27.34\%  & 3.32m,2.58\degree  & 24.77m,15.50\degree & 9.87m,5.89\degree \\

SGLoc+RA (CVPR'23)
& 35.57\% & 29.16\% & 3.27m,3.67\degree & 24.20m,22.70\degree & 9.95m,9.55\degree
& 34.47\% & 26.00\%  & 3.07m,2.82\degree  & \colorbox{third}{19.01m},17.09\degree & 8.92m,6.99\degree \\

LightLoc (CVPR'25)
& \cellcolor{third}59.52\% & \cellcolor{second}56.63\% & \colorbox{third}{2.72m},\colorbox{second}{1.74\degree} & \colorbox{third}{16.22m},15.79\degree & \cellcolor{third}4.07m,4.64\degree
& \cellcolor{third}50.24\% & \cellcolor{third}47.07\%  & 2.77m,2.18\degree  & 30.96m,16.92\degree & 10.29m,6.15\degree \\

RALoc (ICCV'25) 
& 32.08\% & 18.77\%  & 3.86m,3.61\degree & 26.03m,15.14\degree & 9.45m,6.02\degree
& \cellcolor{second}56.25\% & \cellcolor{second}48.90\%  & \colorbox{third}{2.62m},2.19\degree  & \colorbox{second}{13.77m},9.40\degree & \colorbox{second}{4.83m},\colorbox{third}{3.51\degree} \\

\textbf{\method{} (Ours)} 
& \cellcolor{best}\textbf{91.25\%} & \cellcolor{best}\textbf{89.97\%} & \cellcolor{best}\textbf{1.61m,1.20\degree} & \cellcolor{best}\textbf{3.80m,2.92\degree} & \cellcolor{best}\textbf{1.51m,1.04\degree}
& \cellcolor{best}\textbf{94.02\%} & \cellcolor{best}\textbf{91.96\%} & \cellcolor{best}\textbf{1.29m,1.18\degree} & \cellcolor{best}\textbf{2.26m,2.03\degree} & \cellcolor{best}\textbf{1.12m,0.99\degree} \\

\hline
\end{tabular}}
\label{table3}
\vspace{-4mm}
\end{table*}

\begin{figure*}
  \centering
  \includegraphics[width=1\linewidth]{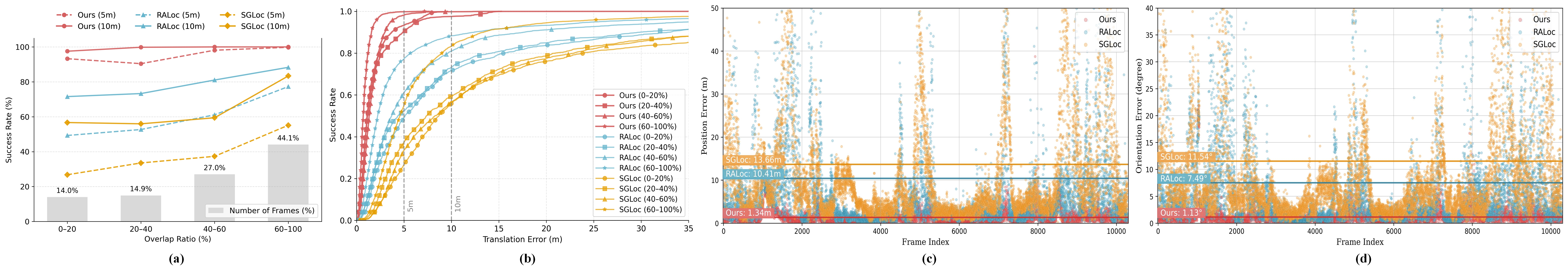}
  \caption{Fig(a) shows success rates (Thr. = 5m\&10m) of three methods across different overlap rate. Fig(b) shows cumulative distribution function (CDF) of three methods under different overlap rate. Fig(c)(d) show distributions of position and rotation errors for three methods, respectively.}
\label{fig6_2}
\end{figure*}

\begin{figure*}
  \centering
  \centering
  \includegraphics[width=1\linewidth]{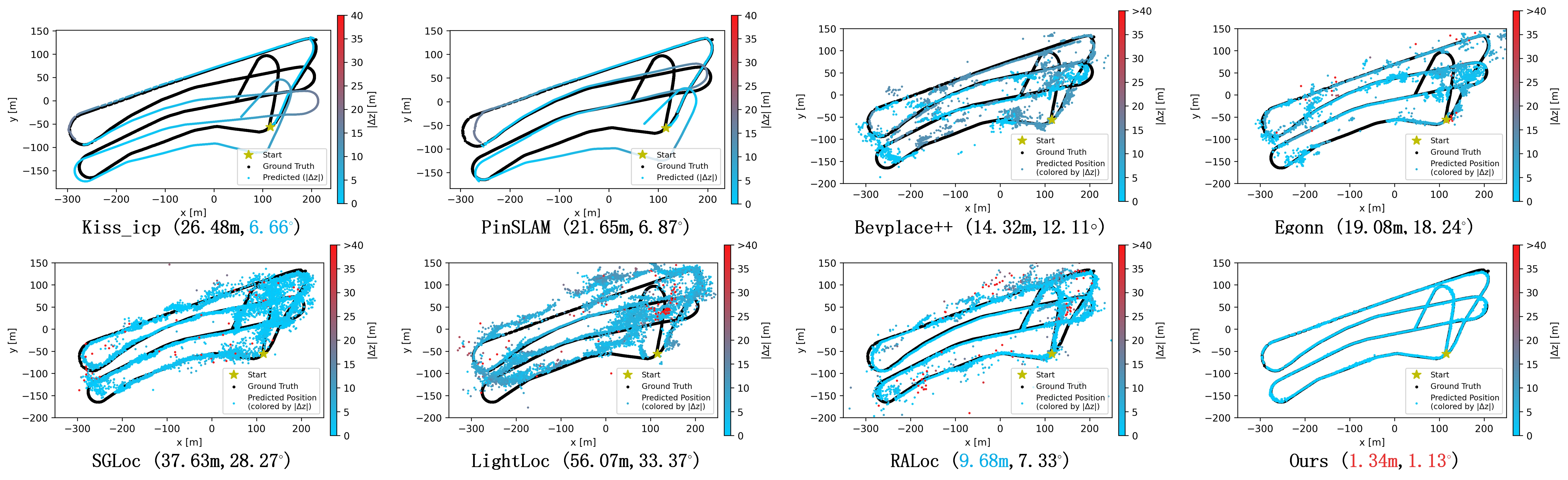}
  \caption{Localization results on Laboratory\_3 (UAVLoc). We highlight \textcolor{red}{best} and \textcolor{blue}{second-best} results.}
\label{fig6}
\end{figure*}

\begin{table}
\centering
\caption{\textbf{Results of Ablation Study,} where w/o means without, and R/ means Replacing.}
\resizebox{1\linewidth}{!}{
\begin{tabular}{l|ccc}
\hline
Modules & AMvalley & Lab\_3 & Lab\_4 \\ \hline
(1) R/ constant features & 5.31m,4.25$^\circ$ & 26.06m,19.93$^\circ$ & 16.28m,12.51$^\circ$ \\
(2) w/o Softmax & 212.13m,105.11$^\circ$ & 156.21m,117.44$^\circ$ & 158.32m,108.09$^\circ$ \\
(3) R/ LoSWAtt & 98.75m,65.32$^\circ$ & 138.19m,64.28$^\circ$ & 164.21m,71.25$^\circ$ \\
(4) w/o Position Encoding & 196.58m,108.72$^\circ$ & 165.92m,116.57$^\circ$ & 168.76m,123.99$^\circ$ \\ \hline
(5) \textbf{Full Method} & \textbf{2.14m,2.15$^\circ$} & \textbf{1.34m,1.13$^\circ$} & \textbf{2.18m,1.89$^\circ$} \\ \hline
\end{tabular}}
\label{table4}
\end{table}

\noindent\textbf{Result on UAVScenes:}
Tab. \ref{table2} reports the comparison between \method{} and existing methods on the UAVScenes dataset. \method{} achieves \textbf{state-of-the-art} performance, yielding average errors by \textbf{[1.39m,2.04\degree]} across four scenes, which substantially surpass those of the comparison methods. Vehicle relocalization methods exhibit significantly degraded accuracy in UAV scenarios. Odometry and SLAM-based methods suffer from high errors when handling large scenes and long trajectories due to cumulative drift. Moreover, due to the high speed of the UAVs, these methods failed on the AMvalley\_3; this is consistent with it presented in original dataset paper~\cite{MARS-LVIG} and is not attributable to incorrect parameter tuning. This further demonstrates the necessity of global relocalization method. Although applying rotation augmentation slightly improved SGLoc performance on the AMtown and AMvalley tracks, it caused a noticeable decline in accuracy on HKairport (2.68 m to 3.69 m). This suggests that data augmentation alone is insufficient to handle the complex and diverse UAV poses, whereas our rotation-invariant feature encoding provides a more robust and generalizable solution.

\noindent\textbf{Result on UAVLoc:} The UAVScenes dataset is insufficient to fully evaluate UAV relocalization methods. Therefore, we collected a more realistic UAV dataset, UAVLoc, for validation. As shown in Tab. \ref{table3}, except for \method{}, all methods suffer from severe performance degradation with many methods exhibiting mean errors above 20–30 meters. The results indicate that irregular trajectories, along with large pose variations, pose significant challenges for all methods. Notably, odometry and SLAM-based methods also fail on Road\_03; this is due to the featureless of road scene and manual flight causing LiDAR data degradation, a similar phenomenon also observed in MARS-LVIG~\cite{MARS-LVIG}. The results show \method{} also achieves \textbf{state-of-the-art} performance, with average errors by \textbf{[2.26m,2.03\degree]} and median errors by \textbf{[1.12m,0.99\degree]}, which significantly outperforms other methods. This improvement is mainly attributed to the robustness of \method{} under low-overlap conditions (see Fig. \ref{fig6_2}(a)) caused by irregular trajectories in UAV relocalization. 

\noindent\textbf{Comparison and analysis of distributions:} As shown in Fig.\ref{fig6_2}, we selected SGLoc+RA, RALoc and ours that performed well on Laboratory\_3 and provided a more detailed visualisation of their results. Fig.\ref{fig6_2}(a) shows the effect of overlap on relocalization performance, where lower overlap rates(0-40\%) increase the difficulty of localization and lead to larger errors. However, \method{} remains markedly more robust than other methods in low-overlap areas. Fig. \ref{fig4}(g) proves that yaw and altitude changes are strongly correlated with a decrease in overlap, thereby verifying the robustness of \method{} against UAV rotation and translation. More failure case analyses and visualizations can be found in the Appendix. Fig.\ref{fig6_2}(b) shows CDF plots on different overlap rate areas, which illustrates \method{} concentrates distribution within 5m error, exhibiting significant superiority. Fig.\ref{fig6_2}(c) and (d) illustrate the distribution of position and rotation errors, showing \method{} exhibits greater robustness.

The above comparison results shows our method achieves state-of-the-art performance on two UAV datasets. Additionally, we visualized the relocalization results of all methods on Laboratory\_3 (UAVLoc) in Fig.\ref{fig6} to provide a more intuitive comparison of their performance.

\noindent\textbf{Ablation Study: } 
As shown in Tab. \ref{table4}, to evaluate each component in \method{}~, we conducted model ablation experiments on UAVScenes (AMvalley) and UAVLoc (Laboratory).

\noindent\textbf{(1) Only replacing constant raw features by XYZ.} In this experiment, we converted the raw features back to classical XYZ relative coordinates. After this conversion, the features no longer remain invariant to rotations and altitude variations. \textbf{(2) Only removing Softmax-free in first LoSWAtt.} Removing Softmax-free module leads to a complete failure of our method, which evidence our design is critical and indispensable. \textbf{(3) Only replacing the LoSWAtt by Point transformer \cite{pointtv3}.} We replace the LoSWAtt with Point Transformer v3, which causes the model to lose its ability to encode locally invariant features, leading to a substantial decline of performance. \textbf{(4) Only removing Position Encoding in LoSWAtt.} Removing position encoding leads to the loss of this local invariance capability, thereby significantly weakening robustness of encoded features.
\section{Conclusion}


In this paper, we identify a critical gap in current research on map-free LiDAR relocalization and show that vehicle-based relocalization methods suffer substantial accuracy degradation when transferred to aerial scenarios. We propose \method{}, a scene coordinate regression-based framework tailored for UAV LiDAR relocalization that is highly robust to both azimuthal and altitude variations. To enable realistic evaluation and foster further progress in this area, we develop a UAV LiDAR system and collect a large-scale dataset with \textbf{irregular trajectories} and \textbf{extensive variations in rotation and altitude}. Extensive experiments demonstrate that our \method{} consistently outperforms existing methods, achieving approximately 1m median errors on two benchmarks.




\maketitlesupplementary

This appendix provides additional details on \method{}, the UAVLoc dataset, and extended experimental analyses.
Section~\ref{sec6} presents implementation details of \method{}, including parameter settings, complexity analysis, and the proposed invariant feature encoding. 
Section~\ref{sec7} provides further details and visualizations of the UAVLoc dataset. 
Section~\ref{sec8} reports additional experimental results, along with visualizations and analyses of failure cases. 
Finally, Section~\ref{fw} discusses potential directions for future work.

\section{\method{} Details}
\label{sec6}

\subsection{Parameter Settings}


We provide detailed parameter settings for \method{}, as shown in Fig. \ref{sup_fig_1}. In the preprocessing stage, we set the voxel size to 0.3 m and the pooling kernel size of the downsampling layer to $k=2$. In the Feature Initialization module, we configure the feature dimension of the Softmax-free LoSWAtt module to 64, the number of attention heads to 2, and the sliding-window size to 8. Subsequently, we use three Feature Encoding modules (i.e., $M=3$) to further encode the features. 
For each module, we set the pooling kernel size to $k=2$. The numbers of LoSWAtt blocks are $N=[2,2,4]$, with feature dimensions of 128, 256, and 512. The corresponding numbers of attention heads are 4, 8, and 16, and the sliding-window sizes are 8, 8, and 16. Finally, we employ six MLP layers (i.e., $T=6$), each with 1024 units, to regress the output coordinates. Detailed procedures are provided in Algorithm \ref{alg:method}. For the final pose estimation, we apply RANSAC \cite{RANSAC} with an inlier threshold of 0.6 m. We use all point-wise correspondences and enforce geometric consistency through edge-length (0.9) and distance-based checks. The RANSAC procedure runs with a maximum of 100,000 iterations and a confidence of 0.999. During training, we use the Adam optimizer~\cite{add_19} with a learning rate of 0.002 and a batch size of 120. More details of the protocol can be found in appendix.

\subsection{Complexity Analysis}

\begin{figure}
  \centering
    \includegraphics[width=1\linewidth]{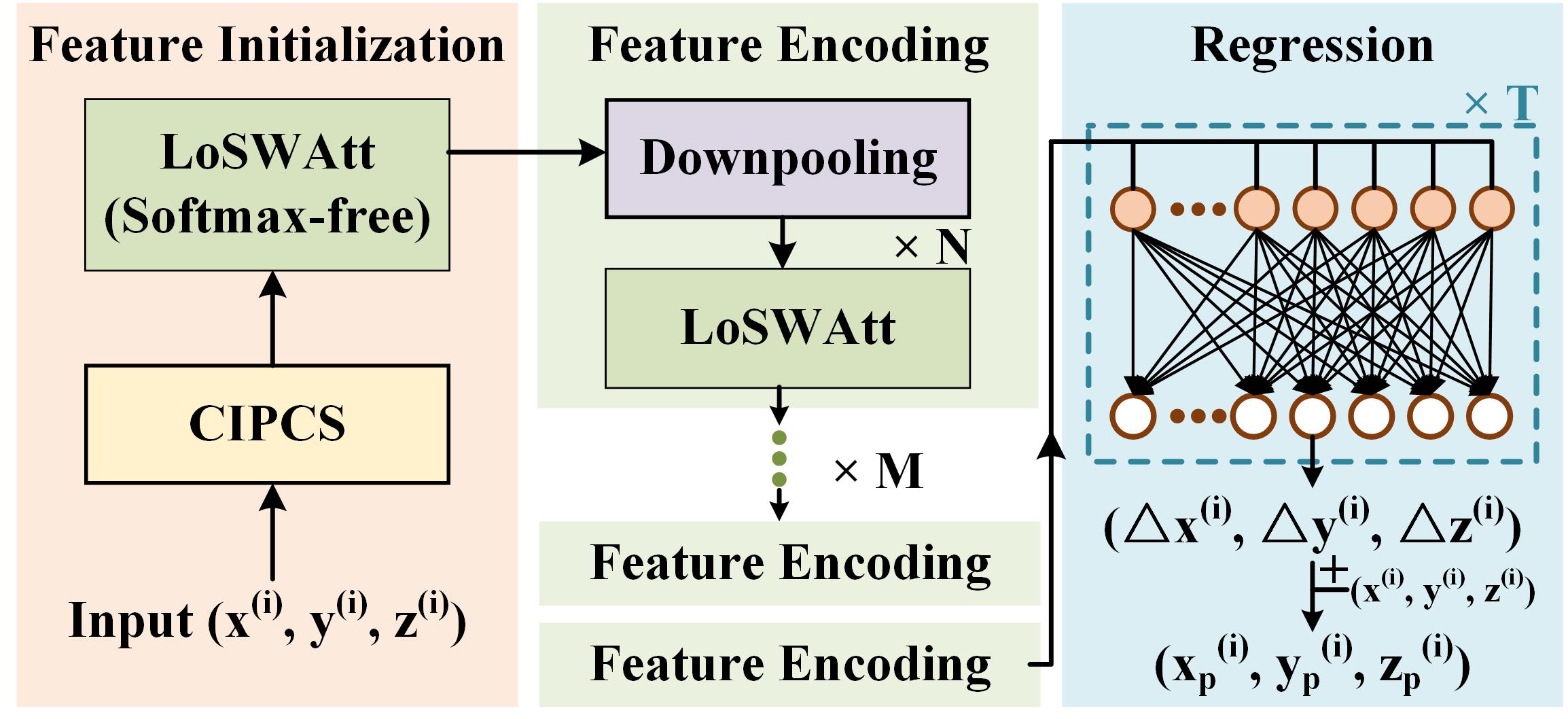}
  \caption{\textbf{The parameter settings of MAILS,} where $N =$ [2, 2, 4], $M =$ 3, $T =$ 6.}
\label{sup_fig_1}
\end{figure}

\begin{algorithm}[t]
\caption{\method{}}
\label{alg:method}
\begin{algorithmic}[1]
\REQUIRE Query point cloud $\mathbf{P}_t \in \mathbb{R}^{N\times3}$
\ENSURE Estimated UAV pose $p^*$

\STATE \textbf{Step 1: CIPCS}
\FOR{$i=1$ to $N$}
    \STATE $\mathbf{f}_i = \varphi(\mathbf{p}_i)$
\ENDFOR
\STATE $\mathbf{F}_{\text{seq}} = \text{Serialize}(\text{Downsample}(\{\mathbf{f}_i\}))$

\STATE \textbf{Step 2: LoSWAtt}
\STATE $N = |\mathbf{F}_{seq}|$
\FOR{$i=1$ to $N$}
    \STATE SW: $s_i = \{j \mid \max(1,i-k) \le j \le \min(\tilde{N}, i+k)\}$
    \STATE $F_i^w = \{\mathbf{f}_j \mid j \in s_i\}$
    \STATE $r_j = \mathbf{p}_j - \mathbf{p}_i, \quad \theta_j = \text{pitch}(\mathbf{p}_j), \quad j \in s_i$
    \STATE $\mathbf{Q}, \mathbf{K}, \mathbf{V} = P_{Q,K,V}(F_i^w), \quad$
    \STATE $\mathbf{Q}_2 = \varphi_2(\mathbf{f}_i)$
    \STATE $\mathbf{K}_2 = \varphi_3 \big[ (r_j, \theta_j) \mid j \in s_i \big]$
        \IF{i == 1}
            \STATE $\mathbf{f}'_i = (\frac{\mathbf{Q}\mathbf{K}^\top}{\sqrt{D}} + \frac{\mathbf{Q}_2 \mathbf{K}_2^\top}{\sqrt{D_2}}) \mathbf{V}$
        \ELSE
            \STATE $\mathbf{f}'_i = \text{Softmax}(\frac{\mathbf{Q}\mathbf{K}^\top}{\sqrt{D}} + \frac{\mathbf{Q}_2 \mathbf{K}_2^\top}{\sqrt{D_2}}) \mathbf{V}$
        \ENDIF
\ENDFOR
\STATE $\mathbf{F}' = \{\mathbf{f}'_i\}_{i=1}^N$

\STATE \textbf{Step 3: Coordinate Regression} \\
\STATE $\mathbf{Y}' = \{\text{MLP}_{\text{reg}}(\mathbf{f}'_i)\}_{i=1}^N$

\STATE \textbf{Step 4: Pose Estimation} \\
\STATE $p^* = \arg\min_{T_p \in \mathbb{R}^{4\times4}} \sum_{i=1}^{M} \|T_p \mathbf{p}_t^{(i)} - \mathbf{y}'_i\|_2$

\STATE \textbf{Step 5: Training Loss} \\
\STATE $\mathcal{L}_{L1} = \frac{1}{|\mathbf{P}_t|}\sum_i \|\mathbf{y}'_i - \mathbf{y}_i^*\|_{1}$

\RETURN $p^*$
\end{algorithmic}
\end{algorithm}

A standard self-attention layer over $n$ points requires computing all pairwise interactions, resulting in a quadratic complexity of $O(n^2C)$, which is prohibitive for UAV LiDAR frames where $n$ often exceeds 10K. Our LoSWAtt module exploits the spatial continuity introduced by space-filling–curve ordering and restricts attention to a local sliding window. For each point, we attend only to its $k$ preceding and $k$ following points along the ordered sequence, forming a window of size $2k{+}1$. Consequently, each query token is involved in at most $2k{+}1$ attention computations, and the computational complexity becomes:
\[
O\big((2k{+}1)nC\big),
\]
which is linear in $n$ and significantly more scalable than global attention. This yields an effective reduction factor of approximately $n/(2k{+}1)$ in the number of pairwise interactions. Meanwhile, the locality constraint introduced by the sliding window further preserves geometric structure, enabling the model to capture yaw- and altitude-robust features.

\subsection{Yaw and Altitude Robustness Encoding}

To analyze the invariance properties of \method{}, we adopt a token-level description that reduces dependence on absolute yaw and altitude. We first define the local geometric token, which is used to construct the positional-bias term $\mathbf{K}_2$ in LoSWAtt.

\begin{definition}[Local Geometric Token]
We initialize point features using a constant input $C$, which removes direct dependence on absolute XYZ coordinates. The resulting features are thus less sensitive to global transformations such as yaw rotations and altitude shifts. For each sliding window centered at index $i$, the local geometric token is
\[
G_i = \{ (r_j, \theta_j) \mid j \in s_i \},
\]
where
\[
s_i = \{j \mid \max(1,i-k) \le j \le \min(N,i+k)\},
\]
\[
r_j = \| (p_j - p_i)_{xy} \|_2,
\]
\[
\theta_j = \arctan\!\left(\frac{z_j - z_i}{\| (p_j - p_i)_{xy} \|_2}\right),
\]
and $p_i$ denotes the raw 3D coordinates. Both $r_j$ and $\theta_j$ depend only on relative geometry, making them approximately invariant to global translations, 
while $\theta_j$ is additionally invariant to yaw rotations.
\end{definition}

\begin{definition}[Feature Initialization]
Let $F_i^w$ denote the features extracted from CIPCS and serialized within the sliding window around $i$. Since the input features are initialized as a constant and processed by shared MLPs, the resulting features do not explicitly encode absolute coordinates, thereby reducing sensitivity to yaw and altitude transformations.
\end{definition}

\paragraph{Assumption (Bounded Perturbation).}
We assume that discretization operations such as voxelization and serialization introduce bounded perturbations under transformation $R$, i.e.,
\[
\|F_i^w(R \circ \mathbf{P}_t) - F_i^w(\mathbf{P}_t)\| \le \delta_1,
\]
\[
\|G_i(R \circ \mathbf{P}_t) - G_i(\mathbf{P}_t)\| \le \delta_2.
\]

\begin{proposition}[Yaw--Altitude Robustness of LoSWAtt]
For the LoSWAtt output at position $i$,
\[
\mathbf{f}_i' = \text{Softmax}\Big(\frac{\mathbf{Q} \mathbf{K}^\top}{\sqrt{D}} + \frac{\mathbf{Q}_2 \mathbf{K}_2^\top}{\sqrt{D_2}}\Big) V,
\]
the mapping $\Psi(F_i^w, G_i)$ is approximately invariant to yaw rotations and robust to altitude changes:
\[
\|\Psi(F_i^w(R \circ \mathbf{P}_t), G_i(R \circ \mathbf{P}_t)) 
- \Psi(F_i^w(\mathbf{P}_t), G_i(\mathbf{P}_t))\| \le \epsilon, 
\]
\[
\forall R \in SO(2)\times \mathbb{R},
\]
where $\epsilon = O(\delta_1 + \delta_2)$ and $R$ denotes a transformation composed of a planar rotation (yaw) and a vertical translation.
\end{proposition}

\begin{proof}
From Definition~1, the local geometric token $G_i$ is constructed using relative distances and pitch angles. The pitch angle $\theta$ is approximately invariant under yaw rotations, and the relative distance $r$ is unaffected by global translation along the vertical axis. Therefore, the positional-bias term $\mathbf{K}_2$ remains stable under such transformations.

From Definition~2, $F_i^w$ does not explicitly depend on absolute coordinates, which reduces its sensitivity to global transformations. Under the bounded perturbation assumption, both $F_i^w$ and $G_i$ change only slightly under $R$.

LoSWAtt consists of:
\begin{enumerate}
    \item a content term $QK^\top$ depending on $F_i^w$,
    \item a positional-bias term $Q_2 \mathbf{K}_2^\top$ depending on $G_i$.
\end{enumerate}
Since both terms are approximately invariant or stable under $R$, and $V$ is a linear transformation of $F_i^w$, the output $\mathbf{f}_i'$ varies smoothly with respect to its inputs. Moreover, the Softmax function is Lipschitz continuous over bounded inputs, which holds in our setting due to normalization by $\sqrt{D}$. Therefore, small perturbations in the inputs lead to bounded changes in the output, yielding the desired result.
\end{proof}

\begin{theorem}[Local Yaw--Altitude Robustness of \method{}]
For the full encoder $\Phi(\mathbf{P}_t)$ (CIPCS + LoSWAtt), and any $R\in SO(2)\times\mathbb{R}$,
\[
\|\Phi(R \circ \mathbf{P}_t) - \Phi(\mathbf{P}_t)\| \le \epsilon.
\]
\end{theorem}

\begin{proof}
Each local output $\mathbf{f}_i'$ is approximately invariant by Proposition~1. Although the sliding window construction and serialization may introduce minor variations under transformation $R$, these perturbations are bounded.

Moreover, each module in $\Phi$ (including MLPs and attention layers) is Lipschitz continuous with bounded constants. Since the composition of Lipschitz functions remains Lipschitz, these deviations do not accumulate significantly across layers.

Therefore, the overall encoder $\Phi(\mathbf{P}_t)$ preserves feature consistency under yaw rotations and altitude changes, leading to robust behavior:
\[
\|\Phi(R \circ \mathbf{P}_t) - \Phi(\mathbf{P}_t)\| \le \epsilon.
\]
\end{proof}

\subsection{Softmax-Free Design Stabilizes Feature Learning}

We provide a theoretical explanation for why the combination of constant feature initialization and the Softmax-free design leads to more stable and discriminative feature learning.

\paragraph{Constant Initialization and Feature Collapse.}
Due to constant initialization, the input features lack spatial diversity.
After shared projections, the resulting features exhibit limited variation:
\[
\mathrm{Var}(Q_i) \approx \epsilon, \quad \epsilon \to 0.
\]
Under this condition, the content-based attention term becomes:
$QK^\top$ exhibits low variance and limited spatial diversity, thus providing weak discriminative power.

\paragraph{Effect of Softmax Normalization.}
With standard attention, the Softmax is applied:
\[
A = \text{Softmax}(QK^\top + Q_2 \mathbf{K}_2^\top).
\]
Although the positional term $Q_2 \mathbf{K}_2^\top$ introduces variation, the Softmax normalizes attention scores into a probability simplex, compressing the dynamic range of attention differences. When the input logits have small variance, i.e., $\mathrm{Var}(QK^\top + Q_2K_2^\top) \to 0$, Softmax produces nearly uniform weights:
\[
A_{ij} \approx \frac{1}{|s_i|}
\]
This results in an \textit{over-smoothing effect}, where features from different points become less distinguishable. Consequently, the output becomes:
\[
\mathbf{f}_i' = A V \approx \text{mean}(V),
\]
which significantly reduces feature discriminability, leading to an over-smoothing effect, where node features converge toward their mean.

\paragraph{Softmax-Free Attention.}
Without Softmax, the positional term contributes additively and remains unsuppressed:
\[
\mathbf{f}_i' = (QK^\top)V + (Q_2 K_2^\top)V.
\]
In this case, the positional term directly contributes to the attention weights without normalization, allowing local geometric variations encoded in $\mathbf{K}_2$ to be preserved. Even when $QK^\top$ is approximately constant, the output is dominated by the positional term:
\[
\mathbf{f}_i' \approx (Q_2 K_2^\top)V.
\]
This mitigates feature collapse and enables the model to learn discriminative local geometric representations. This effect is particularly pronounced when the content term lacks diversity, as in our constant initialization setting.

\paragraph{Impact on Regression Stability.}
Since regression relies on separable feature representations, over-smoothed features lead to ambiguous mappings. With Softmax, the over-smoothed features lead to ambiguous representations and unstable regression. In contrast, the Softmax-free design preserves feature diversity, resulting in more stable and accurate coordinate regression.

This analysis explains the empirical observation in Tab. \ref{table4}, where removing the Softmax leads to substantial performance improvements in our setting.

\begin{table*}[]
\caption{\textbf{Summary of UAVLoc's sequences and their characteristics.}}
\label{summary_table}
\resizebox{1\linewidth}{!}{
\begin{tabular}{cccccc|l}
\hline
Scene names                                      & Sequence                           & LiDAR Frames & Duration (s) & Altitude Range (m) & Size ($km^2$)           & \multicolumn{1}{c}{Scenarios \& Characters}  \\ \hline
\multicolumn{1}{c|}{\multirow{4}{*}{Laboratory}} & \multicolumn{1}{c|}{Laboratory\_1} & 8.0k        &   561        & 48 - 53 & \multirow{4}{*}{$\sim$0.7} & laboratory park (Ascent altitude) \\
\multicolumn{1}{c|}{} & \multicolumn{1}{c|}{Laboratory\_2} & 5.9k & 330  & 44 - 47 &  & laboratory park             \\
\multicolumn{1}{c|}{} & \multicolumn{1}{c|}{Laboratory\_3} & 10.2k & 541 & 46 - 48 &  & laboratory park   \\
\multicolumn{1}{c|}{} & \multicolumn{1}{c|}{Laboratory\_4}  & 8.2k & 472 & 42 - 48 &                            & laboratory park (Reduce altitude) \\ \hline
\multicolumn{1}{c|}{\multirow{3}{*}{Campus}}     & \multicolumn{1}{c|}{Campus\_1}     & 13.4k         & 671 & 60 - 63 & \multirow{3}{*}{$\sim$0.6} & campus area                                  \\
\multicolumn{1}{c|}{}                            & \multicolumn{1}{c|}{Campus\_2}     & 8.6k         & 440 & 67 - 73  &                            & campus area (Ascent altitude)                                  \\
\multicolumn{1}{c|}{}                            & \multicolumn{1}{c|}{Campus\_3}     & 8.8k         & 560 &  53 - 58 &                            & campus area (Reduce altitude)              \\ \hline
\multicolumn{1}{c|}{\multirow{3}{*}{Town}}       & \multicolumn{1}{c|}{Town\_1}       & 9.3k         & 475          & 74 - 80            & \multirow{3}{*}{$\sim$1.0} & rural towns                                  \\
\multicolumn{1}{c|}{}                            & \multicolumn{1}{c|}{Town\_2}       & 10.7k        & 535          & 76 - 82            &                            & rural towns                                  \\
\multicolumn{1}{c|}{}                            & \multicolumn{1}{c|}{Town\_3}       & 10.7k        & 535          & 80 - 85            &                            & rural towns (Ascent altitude)                \\ \hline
\multicolumn{1}{c|}{\multirow{3}{*}{Road}}       & \multicolumn{1}{c|}{Road\_1}       & 8.7k         & 453          & 70 - 76            & \multirow{3}{*}{$\sim$1.7} & Interchanges \& Highways                     \\
\multicolumn{1}{c|}{}                            & \multicolumn{1}{c|}{Road\_2}       & 9.4k         & 472          & 71 - 74            &                            & Interchanges \& Highways                     \\
\multicolumn{1}{c|}{}                            & \multicolumn{1}{c|}{Road\_3}       & 9.8k         & 490          & 75 - 80            &                            & Interchanges \& Highways (Ascent   altitude) \\ \hline \hline
\multicolumn{2}{c|}{Total/Average}                                                    & 121.7k/9.4k  & 6535/503     & 40 - 85            & 4.0/1.0                    & \multicolumn{1}{c}{Irregular trajectories and varying altitudes}                        \\ \hline
\end{tabular}}
\end{table*}


\begin{figure*}
  \centering
    \includegraphics[width=1\linewidth]{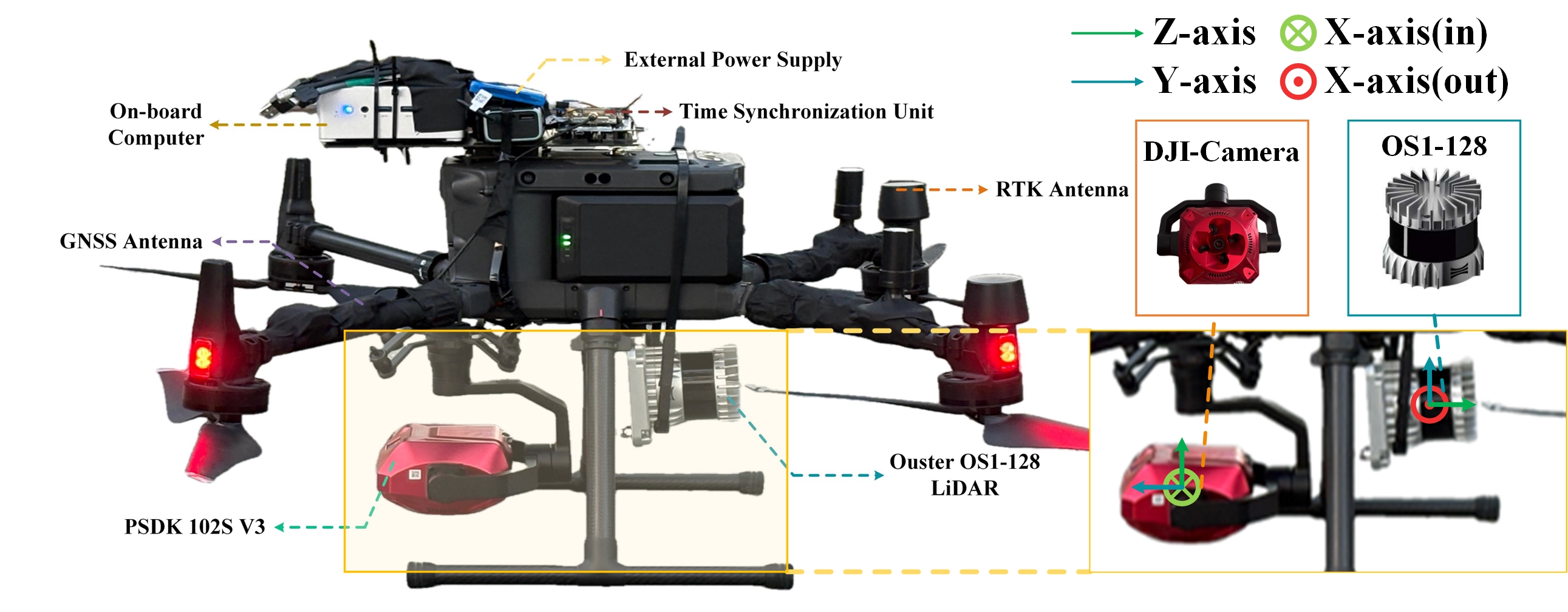}
  \caption{\textbf{Data collection platform and Sensors setup of UAVLoc.}}
\label{uav}
\end{figure*}

\section{UAVLoc Details}
\label{sec7}

Here, we provide additional details of the UAVLoc dataset, and a summary is presented in Tab. \ref{summary_table}. UAVLoc exhibits several distinctive characteristics that make it challenging and representative for UAV relocalization: 1) Multiple flight paths. 2) Irregular flight paths. 3) Altitude variation.

\textbf{Data Collection Platform: }
The data collection platform of UAVLoc is illustrated in Fig. \ref{uav}. All sensors are mounted on a DJI M300 RTK industrial UAV. The primary perception sensors include a PSDK 102 V3 camera and an Ouster OS1-128 LiDAR, which are synchronized using a hardware time-synchronization module. The on-board computer is used to control the UAV and collect the raw LiDAR point cloud data.

\noindent\textbf{Sensors Setup: } 
As shown in Fig. \ref{uav}, our dataset primarily comprises two sensors: (1) an Ouster OS1-128, a 128-beam spinning LiDAR used for point cloud acquisition, with its field of view (FOV), resolution, and frame rate (FPS) configured to [-22.5\degree, 22.5\degree], 20×1024, and 20 Hz, respectively. (2) A DJI camera directly records the UAV’s pose information and refines it through image-based optimization, yielding highly accurate UAV pose. 

\noindent\textbf{Target Environments: }
Our dataset comprises flight data collected across four representative environments: a laboratory park, a campus area, a town and a road area. These environments exhibit diverse structural characteristics and expansive fields of view. Furthermore, non-repetitive flight paths were performed in each environment to evaluate relocalization performance under realistic and unconstrained flight conditions, which are shown in Fig. \ref{fig4}(b). Additionally, to capture altitude variations during UAV operation, we collected at least one flight trajectories with varying altitudes in each scene. This enables the evaluation of relocalization robustness under altitude-varying conditions. Notably, such irregular flight paths and altitude changes pose substantial challenges for relocalization, while more faithfully reflecting UAVs' flight conditions.

\begin{figure*}
  \centering
    \includegraphics[width=1\linewidth]{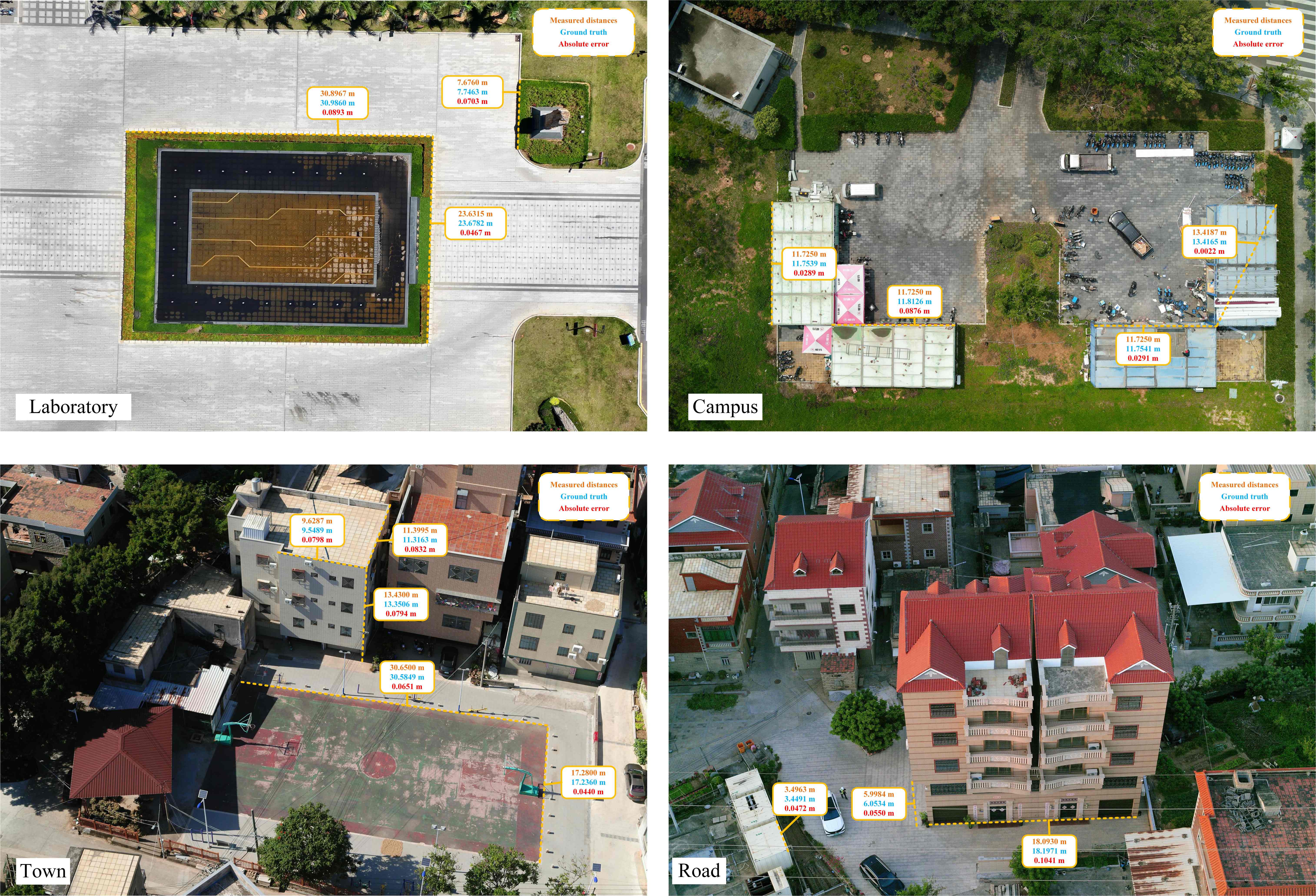}
  \caption{\textbf{Verification of our GT against actual measurement landmarks.} Measured distances refer to physical measurements. Ground truth refers to the values for the corresponding locations obtained from the pose ground truth provided in our dataset. Absolute error refers to the difference between the two values.}
\label{sup_fig_gt}
\end{figure*}

\textbf{Ground Truth Pose: }
In UAVLoc, we use only the LiDAR point clouds for relocalization. The camera data are processed with DJI Terra to obtain high-precision UAV poses, which are then aligned with the LiDAR frames through camera calibration and time synchronization. Moreover, due to inevitable UAV vibrations during flight, following UAVScenes~\cite{UAVScenes} and KITTI~\cite{KITTI}, we further refine the LiDAR poses by performing ICP \cite{ICP} between each LiDAR frame and the high-precision point cloud map generated by DJI Terra, yielding highly accurate LiDAR poses. This process enables us to derive an accurate 6-DoF ground-truth pose for every LiDAR scan.

\begin{figure*}
  \centering
    \includegraphics[width=1\linewidth]{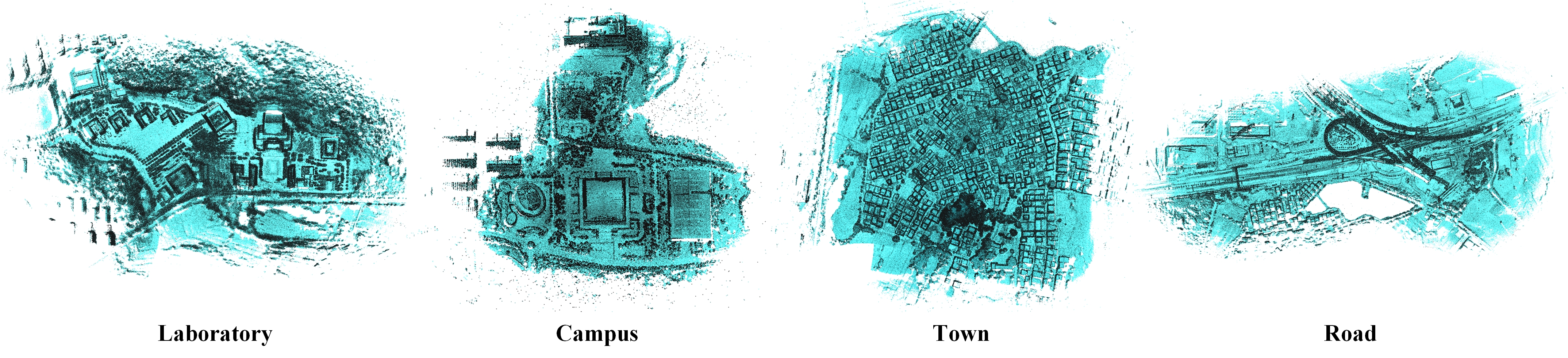}
  \caption{\textbf{Visualization of scene maps.} We used ground truths to reconstruct scene maps, demonstrating the accuracy of our ground truth.}
\label{map}
\end{figure*}

\begin{figure*}
  \centering
    \includegraphics[width=1\linewidth]{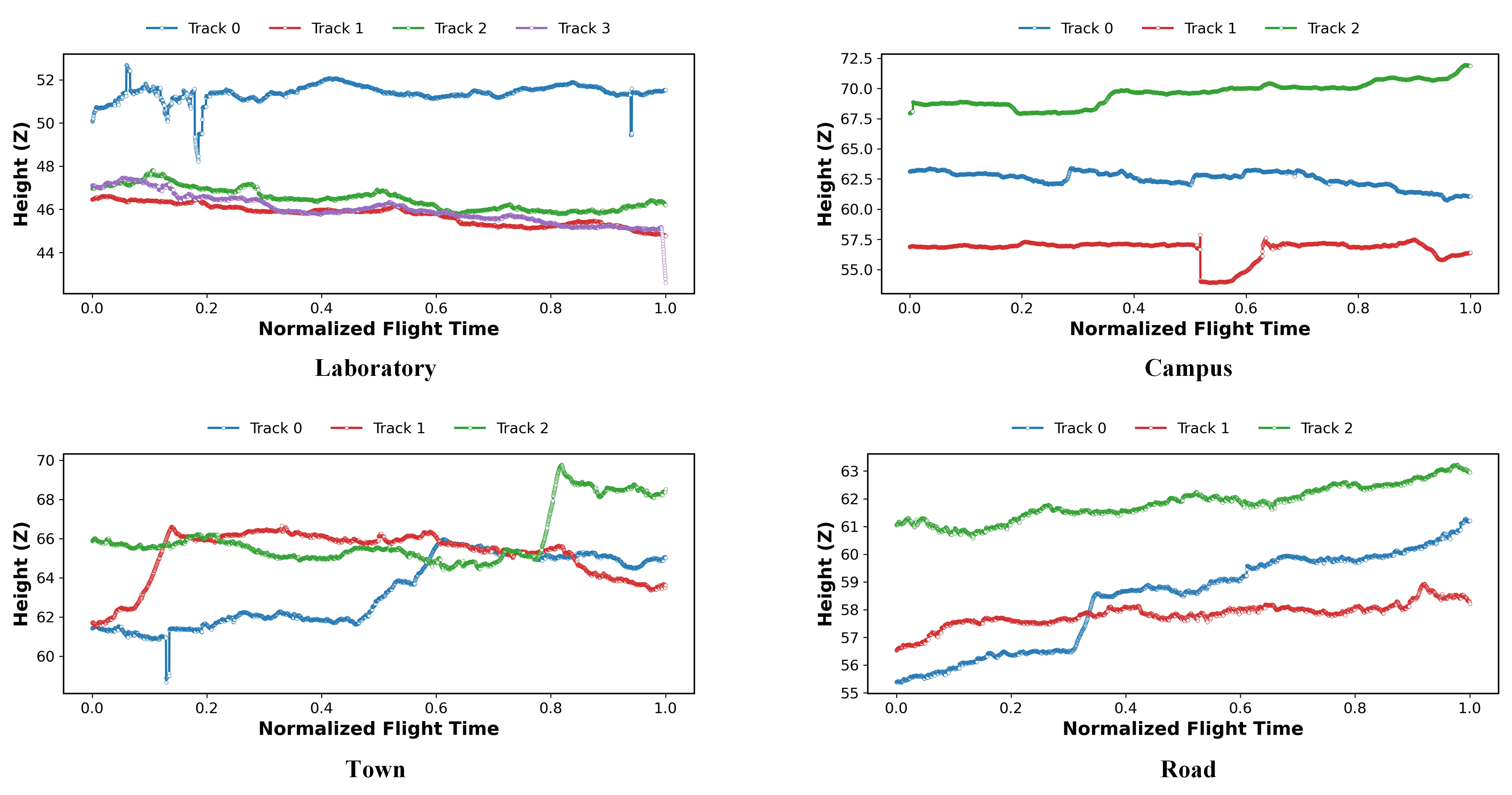}
  \caption{\textbf{Visualization of flight altitudes along different trajectories.} Normalized flight time: scales the flight duration of each trajectory to a common baseline of 1.}
\label{sup_fig_z}
\end{figure*}

\textbf{Accuracy of Ground Truth: }

Additionally, we reconstructed global maps for four scenes using the obtained poses. As shown in Fig. \ref{map}, the reconstructed maps demonstrate the high accuracy of our ground truth. Collectively, these results co nfirm the stability and reliability of our pose estimation pipeline across various environments.

\textbf{Irregular Flight Paths: }
UAVLoc includes multiple flights for each scene. Unlike existing datasets such as UAVScenes \cite{UAVScenes}, each flight in UAVLoc follows a unique and non-repetitive trajectory. As a result, the LiDAR scans from different flights share only partial spatial overlap with the training data, making the relocalization task significantly more challenging. Furthermore, in UAVLoc, the higher flight altitude places the LiDAR near its effective sensing range and reduces the field of view, further decreasing the overlap between scans.

Additionally, unlike UAVScenes, which is one of the most recent relocalization datasets, UAVLoc(ours) explicitly accounts for the impact of altitude variations during UAV flight on relocalization performance. As illustrated in Figure \ref{sup_fig_z}, altitude variations in UAVLoc manifest in two primary aspects: (1) Smooth altitude changes during flight. Due to factors such as air turbulence, UAVs continuously make small altitude corrections during operation. The collected UAVLoc data realistically captures and reflects this natural flight behavior. (2) In real-world operations, UAVs often operate across a wide range of altitudes. Unlike existing datasets such as UAVScenes, which maintain fixed flight altitudes, UAVLoc offers a more challenging and realistic benchmark for evaluating under varying altitude conditions.

\textbf{Licensing and Compliance: }
This work adheres to strict ethical and legal standards for geospatial data. UAVLoc was collected in Xiamen, China under formal approval; we will release it with (1) Licensing: The dataset is released under CC BY-NC-SA 4.0, while the source code follows the Apache 2.0 License. (2) Compliance: Data collection was conducted in compliance with local aviation and surveying regulations in Xiamen, China. All UAV operations were restricted to authorized altitudes and public areas. (3) Privacy: We have implemented a rigorous sanitization pipeline to redact sensitive information, such as random offset of coordinates and to reduce the risk of identifying sensitive infrastructure. (4) Governance: For any compliance inquiries or takedown requests, please contact the corresponding author. We maintain a permanent maintenance policy to respond to such requests promptly.

\section{Additional Experiments}
\label{sec8}

Here, we report the results of additional experiments on the UAVLoc dataset, including cost analysis, data volume, sequence splits, memory efficiency, training and inference time, and robustness evaluations.

\begin{table*}[t]
\centering
\caption{Cost report and sequence split of \method{} on UAVLoc and UAVScenes. }
\label{tab:costreport}
\renewcommand{\arraystretch}{0.95}
\resizebox{1\linewidth}{!}{
\begin{tabular}{cc|cc|cccc}
\hline
\textbf{Dataset} & Scene  & Train seq. & Test seq. & Map size/Model size & Train/Infer time & Frame (train / test) \\ \hline
\multirow{4}{*}{Ours} 
& Laboratory  & 01, 02 & 03, 04 & 5232 MB / \textbf{35.7} MB & 19h / 18ms  & 32,513 (13,980 / 18,533) \\
& Campus     & 01, 02 & 03  & 3676 MB / \textbf{35.7} MB & 20h / 21ms  & 30,783 (21,988 / 8,795) \\
& Town       & 01, 02 & 03  & 4454 MB / \textbf{35.7} MB     & 25h / 15ms & 30,672 (19,970 / 10,702) \\
& Road        & 01, 02 & 03 & 3973 MB / \textbf{35.7} MB    & 27h / 26ms & 27,925 (18,118 / 9,807) \\ \hline
\multirow{4}{*}{UAVScenes} 
& HKairport    & 01, 02 & 03  & 490 MB / \textbf{35.7} MB   & 15h / 19ms & 14,151 (11,128 / 3,023) \\
& HKisland     & 01, 02 & 03  & 211 MB / \textbf{35.7} MB   & 14h / 20ms & 13,450 (10,459 / 2,991) \\
& AMtown      & 01, 02 & 03  & 1187 MB / \textbf{35.7} MB   & 27h / 22ms & 25,442 (19,843 / 5,599) \\
& AMvalley     & 01, 02 & 03  & 469 MB / \textbf{35.7} MB   & 24h / 19ms & 22,431 (17,626 / 4,805) \\ \hline
\end{tabular}}
\end{table*}

\textbf{Cost report and Protocol}

Tab.~\ref{tab:costreport} reports the storage, runtime, and data splits across sequences. Compared to explicit map-based pipelines that require several GB per scene, our method compresses each environment into a compact neural representation (35.7\,MB), reducing storage and communication overhead by two orders of magnitude while maintaining real-time inference (0.1\,s per frame). 

We follow the standard scene-specific training protocol adopted in prior relocalization methods, where models are trained per environment using designated sequences and evaluated on held-out data. The reported training time is a one-time offline cost, while inference is efficient. Although per-scene training is required, this is not unique to SCR-based methods, but also applies to recent retrieval-based approaches (e.g., EgoNN~\cite{egonn}, BEVPlace++~\cite{bevplace++}) and APR-based methods (e.g., DiffLoc~\cite{diffloc}), which likewise require scene adaptation or retraining. In practice, model updates can be handled via incremental fine-tuning rather than full retraining when new data in existing scenes become available.

Importantly, our approach trades explicit map storage for a compact learned representation: while both paradigms are scene-dependent, our method significantly reduces storage and bandwidth requirements, offering a more scalable alternative in resource-constrained scenarios. For completeness, we also report the number of frames used for training and testing, which reflects the data requirement per scene and enables a transparent comparison of system-level costs.

\textbf{Robustness Tests}

We evaluate the sensitivity of our method to voxelization size in Tab.~\ref{tab:voxel_ablation}. Across all evaluated datasets, the performance remains stable under a wide range of voxel sizes (0.2m–0.5m). In particular, the success rate varies within a narrow margin (typically <1\%), and both mean and median errors exhibit only minor fluctuations, indicating that the proposed method is robust to voxel resolution.

We observe that a moderate voxel size (0.3m) consistently achieves the best or near-best performance across all scenes, suggesting a favorable balance between geometric detail preservation and noise suppression. In contrast, smaller voxels (0.2m) introduce slightly higher noise sensitivity, while larger voxels (0.5m) may lead to minor loss of structural details, but neither causes significant degradation.

In addition, we report the inlier ratio at 0.3m, which remains consistently high across different environments (25\%–35\%). This demonstrates that the geometric correspondences established by our method are reliable and well-conditioned for pose estimation.

Overall, these results confirm that the proposed method is robust to voxelization choices and does not require fine-grained parameter tuning for different scenes, which is desirable for practical deployment.

\begin{table*}[t]
\centering
\caption{\textbf{Voxelization sensitivity and inlier ratio.} We report success rate (SR), mean error, and median error under different voxel sizes. The inlier ratio is reported at voxel size = 0.3m.}
\resizebox{1\linewidth}{!}{
\begin{tabular}{l|ccc|ccc|ccc||c}
\hline
\multirow{2}{*}{Datasets} 
& \multicolumn{3}{c|}{Voxel 0.2m} 
& \multicolumn{3}{c|}{Voxel 0.3m} 
& \multicolumn{3}{c||}{Voxel 0.5m} 
& Inlier ratio\\ 
\cline{2-10}

& SR & Mean & Median 
& SR & Mean & Median 
& SR & Mean & Median 
& (voxel 0.3m) \\ 
\hline

AMtown (UAVScenes) 
& 90.81\% & 1.54m,2.91\degree & 0.87m,1.10\degree 
& 91.57\% & 1.34m,2.68\degree & 0.84m,0.90\degree 
& 89.99\% & 1.59m,2.92\degree & 0.92m,1.14\degree 
& 35.38\% \\

Laboratory\_3 (UAVLoc) 
& 95.27\% & 1.51m,1.29\degree & 0.97m,0.91\degree
& 95.98\% & 1.34m,1.13\degree & 0.90m,0.83\degree
& 95.13\% & 1.50m,1.42\degree & 0.96m,1.02\degree 
& 29.57\% \\

Laboratory\_4 (UAVLoc) 
& 86.39\% & 2.34m,1.99\degree & 1.13m,1.05\degree
& 87.34\% & 2.18m,1.89\degree & 0.96m,0.87\degree
& 86.28\% & 2.39m,2.08\degree & 1.15m,1.17\degree 
& 25.11\% \\

\hline
\end{tabular}}
\label{tab:voxel_ablation}
\end{table*}

\begin{table*}[t]
\centering
\caption{\textbf{Effectiveness by LoSWAtt or Feature Engineering.} FE denotes our proposed feature engineering strategy. We report success rate (SR), mean error, and median error on different datasets.}
\resizebox{1\linewidth}{!}{
\begin{tabular}{l|ccc|ccc|ccc}
\hline
\multirow{2}{*}{Method} 
& \multicolumn{3}{c|}{AMtown (UAVScenes)} 
& \multicolumn{3}{c|}{Laboratory\_3 (UAVLoc)} 
& \multicolumn{3}{c}{Laboratory\_4 (UAVLoc)} \\ 
\cline{2-10}

& SR & Mean & Median 
& SR & Mean & Median 
& SR & Mean & Median \\ 
\hline

SGLoc 
& 52.58\% & 12.16m,7.56\degree & 5.77m,2.66\degree
& 32.47\% & 20.08m,15.16\degree & 7.53m,4.33\degree 
& 44.92\% & 23.46m,12.99\degree & 5.61m,2.53\degree \\

SGLoc + FE 
& 67.91\% & 9.22m,4.17\degree & 3.89m,1.93\degree
& 48.33\% & 13.57m,9.49\degree & 5.64m,3.15\degree 
& 58.99\% & 16.18m,9.23\degree & 5.01m,2.29\degree \\ \hline

RALoc 
& 63.78\% & 11.93m,7.11\degree & 3.02m,2.65\degree
& 59.90\% & 9.68m,7.33\degree & 3.38m,2.71\degree
& 51.48\% & 10.78m,9.67\degree & 4.58m,2.93\degree \\

RALoc + FE 
& 78.93\% & 7.66m,5.53\degree & 2.54m,1.97\degree
& 70.31\% & 6.70m,5.71\degree & 2.97m,2.40\degree
& 59.11\% & 8.42m,6.37\degree & 3.17m,2.44\degree \\ \hline

\textbf{Ours} 
& \textbf{91.57\%} & \textbf{1.34m,2.68\degree} & \textbf{0.84m,0.90\degree}
& \textbf{95.98\%} & \textbf{1.34m,1.13\degree} & \textbf{0.90m,0.83\degree}
& \textbf{87.34\%} & \textbf{2.18m,1.89\degree} & \textbf{0.96m,0.87\degree} \\

\hline
\end{tabular}}
\label{tab:fe_ablation}
\end{table*}

\textbf{Effectiveness by LoSWAtt or Feature Engineering}

To disentangle the contribution of the proposed feature engineering (FE) and the LoSWAtt architecture, we also conduct ablation experiments in the main text(Tab. \ref{table4}) to verify each module. Additionally, we conduct controlled experiments by replacing the raw XYZ inputs of existing SCR-based baselines with our FE, while keeping their original network architectures unchanged.

As shown in Tab.~\ref{tab:fe_ablation}, incorporating FE consistently improves the performance of both SGLoc and RALoc across all datasets. For instance, RALoc+FE improves the success rate from 63.78\% to 78.93\% on AMtown and reduces the mean error from 11.93m to 7.66m, demonstrating that the proposed features provide more informative and geometry-aware representations than raw coordinates.

However, despite these gains, all FE-enhanced baselines still remain significantly inferior to our full method. For example, on Laboratory\_3, RALoc+FE achieves 70.31\% SR with 6.70m mean error, whereas our method reaches 95.98\% SR with only 1.34m error. Similar margins are observed across all datasets.

These results indicate that while FE contributes to performance improvement, it is not sufficient to account for the overall gains. The substantial additional improvement of our method is attributed to the LoSWAtt architecture, which better exploits the structured geometric relationships encoded by FE. Therefore, the performance gain of our method arises from the synergy between feature representation and architecture design, rather than from feature engineering alone.

\textbf{Failure Case and Analysis}
\label{fc}

\begin{figure*}
  \centering
    \includegraphics[width=1\linewidth]{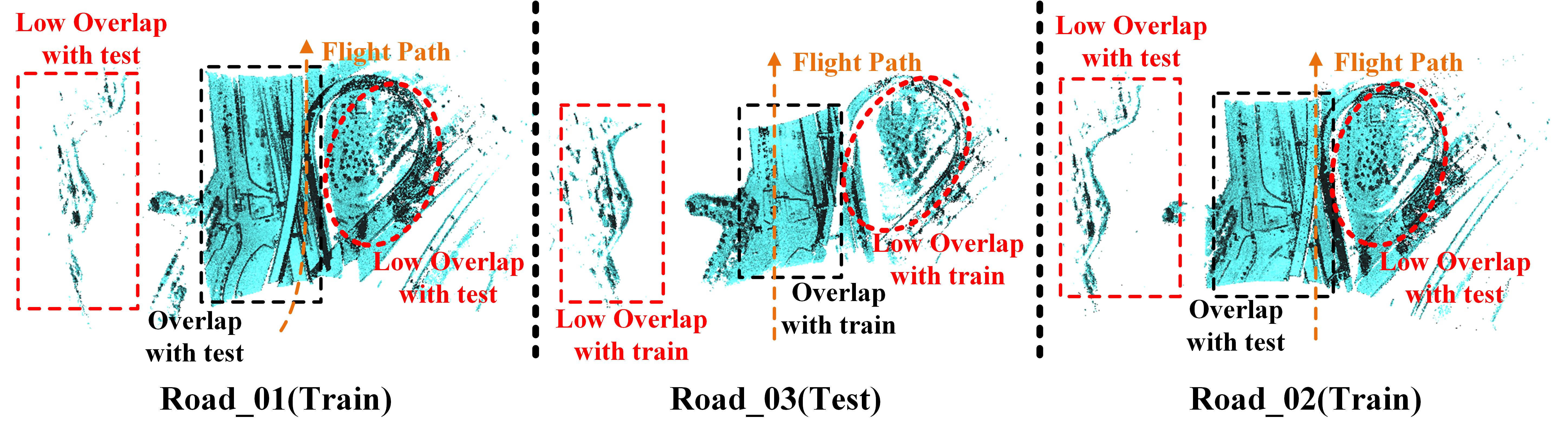}
  \caption{\textbf{Failure cases on UAVLoc (Road sequences).} 
We visualize representative train/test trajectory segments and their spatial overlap. 
Due to irregular UAV flight paths, the overlap between training and testing data can be significantly reduced. 
Regions highlighted in red indicate areas with low overlap, while black boxes denote overlapping regions. 
When the test trajectory deviates from the training distribution, only limited shared structures (e.g., the central roadway) remain, leading to degraded relocalization accuracy.}
\label{fail_2}
\end{figure*}

\begin{figure*}
  \centering
    \includegraphics[width=1\linewidth]{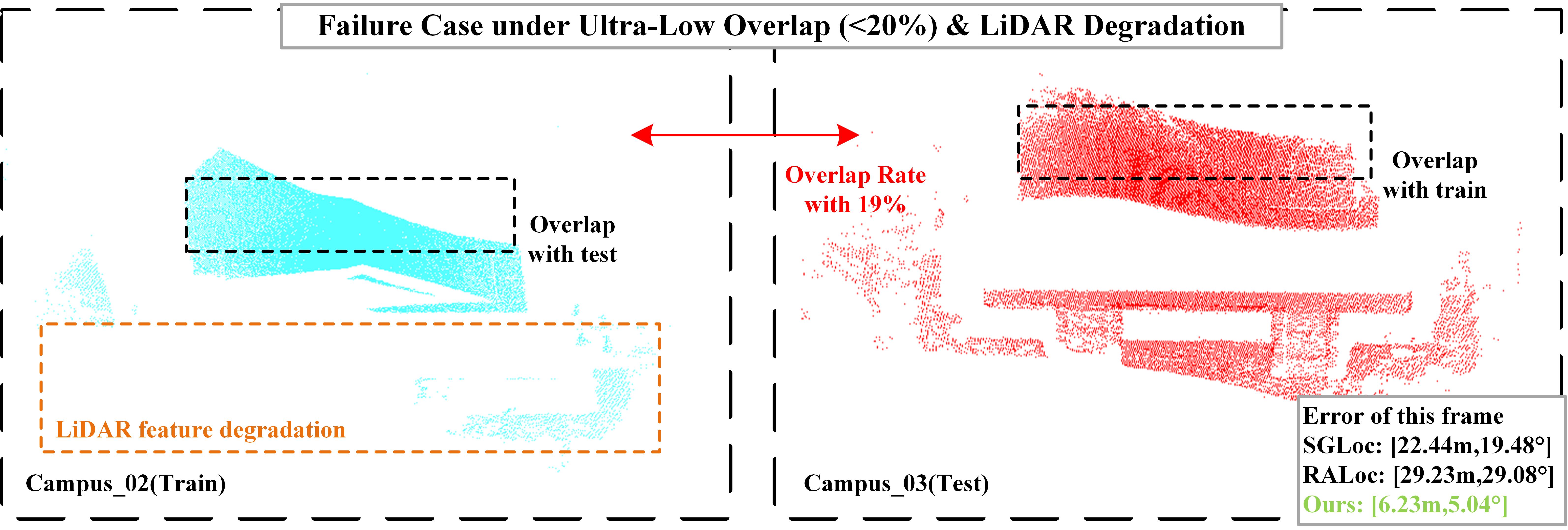}
  \caption{\textbf{Failure cases on UAVLoc (Campus sequences) under ultra-low overlap and LiDAR degradation.}}
\label{fail_1}
\end{figure*}

As shown in Fig.~\ref{fail_2}, we analyze failure cases on UAVLoc (Road sequences) to understand the limitations of map-free UAV relocalization under extreme conditions. Our analysis reveals a clear causal chain: UAV motion variations (e.g., yaw and altitude changes) lead to reduced spatial overlap between training and testing trajectories, which is strongly correlated with a reduced number of shared geometric structures and a corresponding degradation in matching reliability.

Fig.~\ref{fig4}(g) further confirms that larger yaw and altitude variations strongly correlate with lower overlap, indicating that irregular UAV flight patterns are a primary source of distribution shift in aerial relocalization scenarios. This also implies that viewpoint and scale changes induced by UAV motion are not independent factors, but jointly contribute to the collapse of spatial overlap. 

As shown in Fig.~\ref{fig6_2}(a), this reduction in overlap leads to a consistent and significant increase in relocalization error across all methods, particularly in the 0--40\% overlap regime, where the number of valid geometric correspondences becomes extremely limited and matching ambiguity increases sharply. In this low-overlap regime, only a small subset of stable structures (e.g., road surfaces and large planar regions) remains consistently observable across trajectories, while most contextual structures (e.g., buildings, roadside objects, and background geometry) become non-overlapping, as illustrated in Fig.~\ref{fail_2}. This structural sparsification significantly reduces the density and diversity of reliable geometric cues, making pose recovery highly under-constrained. Although our \method{} cannot fully resolve this ambiguity due to the intrinsic lack of sufficient shared structure, it remains more robust than competing methods by better exploiting sparse and locally consistent geometric patterns under low-overlap conditions.

To provide a comprehensive understanding of our method's boundaries, we also analyze representative failure cases from the Campus sequences in UAVLoc, as illustrated in Fig.~\ref{fail_1} which faces \textbf{ultra-low spatial overlap} coupled with \textbf{geometric degeneration}. As visualized in Fig.~\ref{fail_1}, the test UAV executes a severe shift, where the overlap ratio drops below $20\%$. From a geometric perspective, this shift manifests as a collapse of observable shared structures. Meanwhile, the high-altitude viewpoint and significant lateral displacement cause most discriminative vertical structures (e.g., building facades) to be occluded. In this failure cases, three methods SGLoc, RALoc, and \method{} all exhibited increased errors at varying levels. However, \method{} demonstrated greater robustness in this extreme scenario, with an error of [6.23 m, 5.04\degree]. This suggests that our \textit{Softmax-free design} and \textit{Coordinate-independent Initialization} effectively prevent the model from outputting catastrophic outliers by maintaining a focus on the few remaining stable local geometric patterns, even when a global consensus is difficult to reach. The error distributions in Fig. \ref{fig6_2}(c) and (d) also demonstrates the robustness of \method{} in such situations.

\section{Future Work and Discussion}
\label{fw}

The above analysis highlights several fundamental challenges in UAV relocalization, which also point to important directions for future improvements. 
We discuss these aspects from the perspectives of representation, generalization, and system-level formulation.

\textbf{Map-free formulation and implicit scene representation.}
In this work, we emphasize that \textbf{``map-free''} here specifically refers to the absence of explicit geometric structures at inference time, rather than the absence of learned scene priors, following prior works~\cite{map_free,RALoc,SGLoc}. 
The scene prior is implicitly encoded in the network parameters through scene-specific training. 
Therefore, our formulation is more accurately viewed as an \textbf{implicit-map-based localization paradigm}, rather than a strict elimination of environmental priors.

From a broader perspective, our method belongs to the family of \textbf{scene coordinate regression (SCR)} approaches, which can be interpreted as learning an \emph{implicit scene representation} that maps observations to 3D geometry. 
In this sense, SCR methods inherently construct an \textbf{implicit neural map}, even though no explicit map structure is materialized.

This establishes a clearer relationship between different paradigms:
\emph{explicit mapping} (e.g., LiDAR SLAM) maintains geometric structures such as point clouds or voxel grids;
\emph{implicit mapping} (e.g., PIN-SLAM~\cite{pinslam} and SCR-based methods) encodes the scene in neural representations;
and \emph{map-free inference} refers specifically to the fact that no explicit map needs to be stored or accessed at test time.

Under this view, these paradigms are not mutually exclusive categories.
For example, recent works such as PIN-SLAM~\cite{pinslam} jointly learn implicit neural maps while performing localization, illustrating that the boundary between mapping and localization is increasingly blurred. 
Our approach avoids maintaining an explicit or continuously updated map at inference time. Moreover, as shown in Tab.~\ref{table2} and Tab.~\ref{table3}, our method achieves competitive or stronger performance in global relocalization compared to PIN-SLAM.
We attribute this difference to the fact that SLAM-based systems may be affected by drift accumulation in large-scale or weakly constrained environments, whereas our formulation directly optimizes for global relocalization accuracy.

Although this comes with a one-time per-scene training cost, this requirement is shared by many existing relocalization paradigms, including scene coordinate regression (SCR), retrieval-based methods (e.g., EgoNN~\cite{egonn}, BEVPlace++~\cite{bevplace++}), and absolute pose regression approaches (e.g., DiffLoc~\cite{diffloc}), all of which rely on scene-specific adaptation or retraining.

In practice, such costs can be amortized, as models can be incrementally updated when new data becomes available, rather than retrained from scratch. Prior works (e.g., LightLoc~\cite{Lightloc}) have explored such incremental update schemes. 
Moreover, as shown in Tab.~\ref{tab:costreport}, our method already provides a favorable tradeoff by compressing each scene into a compact neural representation (35.7\,MB), significantly reducing storage and communication overhead compared to explicit map-based pipelines.

Meanwhile, as shown in \ref{tab:costreport}, \method{} achieves an inference latency of approximately 20 ms on an NVIDIA GeForce RTX 5090 GPU. To further evaluate its computational efficiency under more realistic deployment settings, we also report results on an NVIDIA GeForce RTX 3090 GPU, where it achieves an average inference time of 35 ms on the UAVLoc dataset. These results indicate that our method maintains reasonable computational efficiency under different hardware configurations, which is important for potential UAV edge deployment scenarios with limited computational resources. However, we note that real-world UAV platforms may exhibit more diverse and stricter hardware constraints. In future work, we will explore further model compression and optimization techniques, and evaluate the method on embedded or onboard UAV computing platforms.

We also note that practical systems may adopt intermediate designs between explicit and implicit representations.
For instance, methods such as OPAL~\cite{Opal} leverage lightweight geographic priors (e.g., OpenStreetMap) to enable efficient localization from sparse observations.
However, such approaches rely on the availability and quality of external map priors, which may be incomplete, outdated, or unavailable in certain environments (e.g., rural or rapidly changing areas), and may introduce additional assumptions on data sources.
In contrast, our method does not depend on any external geographic prior and instead learns scene representations directly from sensor data, making it more generally applicable across diverse environments.
Exploring hybrid strategies that combine lightweight priors with learned representations may further improve the balance between efficiency, scalability, and generalization.

\textbf{Robustness to full 6-DoF transformations.}
While our current design explicitly models yaw and altitude variations through relative geometric encoding, UAV platforms inherently undergo full 6-DoF motion, including roll and pitch. 
Our formulation achieves partial robustness in practice because local geometric relations (e.g., relative distances and pitch angles) remain stable under moderate viewpoint changes. 
However, it does not explicitly enforce equivariance over the full $\mathrm{SO}(3)$ space.

As a result, performance may degrade under large roll/pitch variations or aggressive maneuvers. 
To address this limitation, future work will explore rotation-equivariant representations and canonicalization strategies, enabling principled handling of arbitrary 3D transformations and improving robustness in highly dynamic UAV scenarios.

\textbf{Generalization under limited spatial overlap.}
As demonstrated in the failure cases, relocalization performance degrades when the overlap between training and testing trajectories becomes extremely limited. 
This issue stems from the inherent distribution shift introduced by irregular UAV flight paths.

While our method partially alleviates this issue through local geometric encoding and structured attention, it still relies on sufficient shared spatial context. 
Future work will investigate mechanisms for enhancing global structural reasoning, such as cross-view context aggregation and long-range geometric consistency modeling, allowing the model to infer correspondences even under sparse or asymmetric overlap.

\textbf{Learning under broader trajectory distributions.}
Finally, our current training protocol remains scene-specific and trajectory-dependent. 
To improve generalization, we will explore self-supervised pretraining and cross-scene learning strategies that expose the model to more diverse environments, viewpoints, and motion patterns. 
This may reduce reliance on dense per-scene supervision and improve robustness to unseen trajectories.

Overall, these directions aim to better bridge implicit and explicit representations, improve robustness under full 6-DoF motion, and enhance generalization in real-world UAV relocalization systems.

\end{document}